\documentclass{article}

\usepackage[final]{neurips_2024}

\usepackage[utf8]{inputenc} 
\usepackage[T1]{fontenc}    
\usepackage{hyperref}       
\usepackage{url}            
\usepackage{booktabs}       
\usepackage{amsfonts}       
\usepackage{nicefrac}       
\usepackage{microtype}      
\usepackage{xcolor}         

\usepackage{bbm, dsfont}
\usepackage{amsmath}
\usepackage{amssymb}
\usepackage{mathtools}
\usepackage{amsthm}
\usepackage{graphicx}
\usepackage{wrapfig,lipsum,booktabs}

\DeclareMathOperator{\Tr}{Tr}

\theoremstyle{plain}
\newtheorem{theorem}{Theorem}[section]
\newtheorem{proposition}[theorem]{Proposition}
\newtheorem{lemma}[theorem]{Lemma}

\theoremstyle{definition}

\theoremstyle{remark}

\title{Weight decay induces low-rank attention layers}

\author{%
  Seijin Kobayashi\thanks{Equal contribution, order determined randomly}~ \textsuperscript{1,2}, Yassir Akram$^*$\textsuperscript{1}\\
  \textbf{Johannes von Oswald\textsuperscript{2}}\\
  \\
  \textsuperscript{1}Department of Computer Science, 
  ETH Z\"{u}rich\\
  \textsuperscript{2}Google, Paradigms of Intelligence Team\\
  \texttt{\{seijink, jvoswald\}@google.com,  yakram@ethz.ch}  
}

\begin{document}

\maketitle

\begin{abstract}
The effect of regularizers such as weight decay when training deep neural networks is not well understood. We study the influence of weight decay as well as $L2$-regularization when training neural network models in which parameter matrices interact multiplicatively. This combination is of particular interest as this parametrization is common in attention layers, the workhorse of transformers. Here, key-query, as well as value-projection parameter matrices, are multiplied directly with each other: $W_K^TW_Q$ and $PW_V$. 
We extend previous results and show on one hand that any local minimum of a $L2$-regularized loss of the form $L(AB^\top) + \lambda (\|A\|^2 + \|B\|^2)$ coincides with a minimum of the nuclear norm-regularized loss $L(AB^\top) + \lambda\|AB^\top\|_*$, and on the other hand that the 2 losses become identical exponentially quickly during training. We thus complement existing works linking $L2$-regularization with low-rank regularization, and in particular, explain why such regularization on the matrix product affects early stages of training.
Based on these theoretical insights, we verify empirically that the key-query and value-projection matrix products $W_K^TW_Q, PW_V$ within attention layers, when optimized with weight decay, as usually done in vision tasks and language modelling, indeed induce a significant reduction in the rank of $W_K^TW_Q$ and $PW_V$, even in fully online training.
We find that, in accordance with existing work, inducing low rank in attention matrix products can damage language model performance, and observe advantages when decoupling weight decay in attention layers from the rest of the parameters.
\end{abstract}

\section{Introduction}

The influence of $L2$-regularization, as well as \textit{weight decay} regularization when training deep neural network models remains poorly understood and is still a subject of active research \citep{vanlaarhoven2017l2, 10.1145/3446776, zhang2018three, loshchilov2018decoupled, 10.1145/3446776, xie2023on, andriushchenko2023need}. Given a model parametrized by matrix $W$, the standard motivation of adding  $\frac{\lambda}{2}\|W\|^2$ to the optimization loss $L(W)$ comes from framing learning the model weights $W$ as maximum a posteriori (MAP) estimation and choosing a Gaussian prior with zero mean \citep{Mackay, Krogh1991ASW}.

Previous works have studied the effect of regularization on the rank of weight matrices when training a model with gradient-based optimization \citep{pmlr-v202-ziyin23a, DBLP:journals/corr/abs-1905-13655, li2021towards, razin2020implicit,gunasekar2017implicit}. Here, we focus on the effect of $L2$-regularization on models using a \textit{factorized} parametrization, where some weight matrices are parametrized as products of (often lower rank) matrices, $W=AB^\top$. This parametrization is used heavily in attention layers inside transformers \citep{transformers} which we will focus on in the following.

Indeed, at the heart of the Transformer architecture is the attention operation which updates the $T$ \textit{tokens} concatenated into a matrix $E \in \mathbb{R}^{d_m \times T}$ inside the network according to
\begin{equation}
    E \leftarrow E +  P W_{V} E \phi\big((E^\top W_{K}^\top W_{Q} E) \odot M\big) 
\end{equation}
where $\phi$ is typically a softmax operation applied column-wise and $M$ is typically a causal mask. The matrices $W_V, W_K, W_Q \in \mathbb{R}^{d_k \times d_m}$ are respectively the value, key, and query matrices that linearly transform $E$ into some typically smaller space of dimension $d_k$ \citep{phuong2022formal}, which can potentially subsume bias terms by appending a constant $1$ to the tokens. The weight matrix $P \in \mathbb{R}^{d_m \times d_k}$ projects the weighted sum of value vectors back into the original token dimension. Therefore (multi-head) attention layers indeed consist of parameter matrix products i.e. $W_{QK}=W_{K}^\top W_{Q}$ as well as $W_{VP}=PW_{V}$, regardless of the choice of $\phi$, or the presence or absence of causal masks.

When optimizing neural network models with this particular parametrization in conjunction with $L2$-regularization, and for any such two weight matrices $A$ and $B$ (e.g. the $P$ and $W_V$ for a given layer and a given head), we can rewrite the loss as:
\begin{equation}\label{eq:fullregloss}
    \mathcal{L}_{L2}(A,B, \theta) \coloneqq L(AB^\top, \theta) + \frac{\lambda}{2}(\|A\|^2 + \|B\|^2),
\end{equation}
where $\theta$ accounts for all the remaining parameters. We will see in the following that optimizing such losses has in practice implications on regularizing the rank of $W=AB^\top$. In fact, while it is classically known that the summed Frobenius norm $\frac{1}{2}(\|A\|^2 + \|B\|^2)$ is a tight upper bound on the \textit{nuclear norm} $\|AB^\top\|_*$ \citep{rank_srebro, tibshirani2021equivalences}, we theoretically show in the following that gradient-based optimization of the above objective result in the upper bound becoming tight exponentially quickly, for arbitrary loss, and thus directly optimizes for the nuclear norm which is known to induce low rank.

We highlight the relevance of this study since high weight decay is commonly used when training Transformer models. For example, GPT-3 \citep{transformers_few_shot}, LLaMa \citep{touvron2023llama}, LLaMa 2 \citep{touvron2023llama} and ViT \citep{dosovitskiy2021an} report a weight decay strength of $\lambda=0.1$. Interestingly, this is even true when fine-tuning, for example with low-rank adaptation (LoRA) \citep{lora}.

We summarize our contributions below:
\begin{itemize}

    \item We show that for models with factorized parametrization, all local minima of any loss regularized by the Frobenius norm of $A,B$ coincide with local minima of the same loss regularized by the nuclear norm of $W$. We further show theoretically that the discrepancy between the 2 regularizations vanishes exponentially quickly during training, thus implying that training such models with weight regularization can be subjected to low rank inducing pressure long before convergence. 
    \item We empirically validate our result on various experimental settings, including when optimization with decoupled weight decay \citep{loshchilov2018decoupled}, on models ranging from deep linear networks to language models as well as Vision Transformers \cite{dosovitskiy2021an}. Intriguingly, we observe that this inductive bias of factorized parametrization with weight decay seems to hurt the performance on some tasks, raising the question of whether it is a feature or a bug.
    \item We provide evidence suggesting that this rank-regularizing effect in fact seems to affect the pretraining of popular pre-trained foundation models such as LLAMA 2 \citep{touvron2023llama} and Vision Transformer \citep{wu2020visual}, by analyzing their pre-trained weights.
\end{itemize}

\section{Related Work}
The setting we study is closely related to a setting extensively studied in the Matrix Completion literature \citep{rank_srebro,Sun_2016,candes2009power}, where the goal is to recover an unknown low-rank matrix for which only a subset of its entries are specified. Nuclear norm regularization is often used as a convex relaxation of the problem \citep{HU2021218}, and its equivalence at the global optimum with the $L2$-regularization on factorized matrix \citep{rank_srebro}, which has the advantage of being differentiable everywhere, has been exploited as a popular approach for large-scale matrix completion. Extensive prior work has focused in this setting on the theoretical guarantee of the factorization formulation to recover the underlying low-rank matrix correctly \citep{Sun_2016,candes2009power}. Similarly, similar loss landscape analyses were performed in the context of unconstrained features models \citep{zhu2021geometric}. In contrast, our analysis does not rely on assumptions about the data, the loss (other than its differentiability) or convergence.

In a different line of work, recent efforts have focused on the effect of gradient-based optimization of deep networks on the parametrized matrix. For example, small weight initialization in this setting was shown to induce low rank in deep linear networks \citep{jacot2022saddletosaddle,DBLP:journals/corr/abs-1905-13655,li2021towards}. More recently, \citep{jacot2023implicit} has shown the representation cost of deep networks with homogeneous nonlinearity converges to a notion of rank over nonlinear functions. More related to our work, equivalence between $L2$ regularization applied on factorized matrices and a low-rank inducing $L_p$-Schatten norm on the matrix they parametrize has been shown in several prior works \citep{NEURIPS2021_e22cb9d6,tibshirani2021equivalences}. This is particularly relevant as $L2$ regularization can be applied explicitly or implicitly, such as when training deep networks with homogeneous activation coupled with e.g. the cross entropy loss \citep{jacot2022saddletosaddle,DBLP:journals/corr/abs-1905-13655}. Crucially, however, these existing works characterize the low-rank inducing bias on neural networks that globally minimize $L2$ regularization while fitting training data. 

Recently, \citep{galanti2023characterizing} have studied the effect of SGD  with L2-regularization on a general architecture. Similarly to our work, they consider a general differentiable loss, but bound the rank of matrices at sufficiently large training steps, employing a theoretical argument that crucially does not leverage low-rank inducing norms due in part to the generality of the architecture they consider. \citep{wang2023implicit} have studied the same effect in the context of deep fully connected linear networks, showing that SGD strengthens the already existing low-rank bias induced by L2-regularization, albeit on matrix completion problems. Similarly to our work, they draw for the first time, to the best of our knowledge, an equivalence between the critical points of L2-regularized loss on the factorized matrix and Nuclear norm regularized loss on the parametrized matrix. 

In contrast to these past works, we show both theoretically and empirically that for any arbitrary differentiable loss, the two regularizations become exponentially quickly identical during gradient-based optimization, and thus, that the low-rank inducing effect comes into play very early in during training. This brings a theoretical understanding to empirical observations made in previous works \citep{khodak2022initialization}, and is particularly relevant for many practical settings, in which learning does not converge, such as foundation model trained online, as is commonly done for large language models (LLMs) and large vision models. 

Finally, given the significance of self-attention models, there has been work trying to understand the implicit inductive biases of some of their design choices. \citep{bhojanapalli2020lowrank} shows, in particular, the head size heuristic commonly used causes a low-rank bottleneck and limits the expressive power of the multi-head attention layer. Recent work has shown indeed that reducing the rank of attention matrices post-training of LLMs can hurt downstream performance \citep{sharma2023truth}. Our empirical work complements these observations and sheds light on the potentially damaging effect of the implicit rank-reducing effect of weight decay in the context of Attention layers, an unintended side effect contrary to the matrix completion setting.

\section{Theoretical results}

\subsection{Preliminaries}
We begin by reviewing the definition of the nuclear norm of a matrix and its upper bound when applied to a factorized matrix. We denote by $\|\cdot\|$ the Frobenius norm when applied on matrices.

\subsubsection{Nuclear norm}
The nuclear norm (also known as trace norm) of a real-valued matrix $W$, denoted by $\|W\|_*$, is defined as
\begin{equation}
\|W\|_* = \Tr(\sqrt{WW^\top})
\end{equation}

When using the singular value decomposition (SVD) of $W$, $W=USV^\top$, denoting $(s_i)_i$ the singular values, we can see that 
\begin{equation}
\|W\|_* = \Tr(\sqrt{USV^\top V S U^\top}) = \Tr(S) = \sum_i s_i
\end{equation}
i.e. the nuclear norm is the sum of the singular values of $W$.

The nuclear norm is often used in the low-rank regularization literature \citep{HU2021218} as it intuitively is a convex relaxation of the rank, and regularizing it typically induces low rank by injecting sparsity in the singular values.

\subsubsection{Upper bound of the nuclear norm of a factorized matrix}

Let two matrices $A,B$ such that $W=AB^\top$. Then, using the Cauchy-Schwarz inequality, we have that
\begin{align} \label{eq:gap}
\|W\|_* = \Tr(S) &= \Tr(U^\top AB^\top V) \\
&\leq \sqrt{\Tr(U^\top AA^\top U) \Tr(B^\top VV^\top B)} \\
&= \|A\| \|B\| \leq \frac{1}{2}(\|A\|^2 + \|B\|^2)
\end{align}

\subsubsection{Considered losses}
We will consider $L2$ losses of the format
\begin{equation}\label{eq:regloss}
    \mathcal{L}_{L2}(A,B) \coloneqq L(AB^\top) + \frac{\lambda}{2}(\|A\|^2 + \|B\|^2),
\end{equation}
and their $L_\star$ counterpart
\begin{equation}\label{eq:regloss_nuclear}
    \mathcal{L}_*(AB^\top) \coloneqq L(AB^\top) + \lambda\|AB^\top\|_*.
\end{equation}

As a consequence of the above inequality, the $L2$-regularized objective \eqref{eq:regloss} is an upper bound of the nuclear norm-regularized objective.

The meticulous reader should spot that those objectives don't account for the remaining parameters $\theta$ as in \eqref{eq:fullregloss}, while those parameters also evolve through learning. In fact, one can convince oneself that this can be safely ignored without loss of generality. The reader is referred to appendix \ref{linktwolosses} for more details about this point.

\subsection{Equivalence of optimization solution}\label{sec:l2_equilibrium}

In the following, we will first show that in fact, any objective of the form in \eqref{eq:regloss} will coincide at any stationary point with the nuclear-norm regularized loss in \eqref{eq:regloss_nuclear}, thus introducing a low-rank inducing bias in the solution found. We assume $A,B$ to have a bottleneck, i.e. to have the number of rows greater or equal to the number of columns, as is usual in attention layers.  All proofs can be found in Appendix~\ref{app:proof}.

We start by providing a sufficient condition under which the averaged Frobenius norm of two matrices would correspond to the nuclear norm of their product.

\begin{proposition} \label{prop:norm_eq}
Let $A,B$ be matrices such that $A^\top A=B^\top B$. Then, denoting $AB^\top = U S V^\top$ the SVD of $AB^\top$, there exist an orthogonal matrix $O$ such that $A = U \left(\begin{array}{c} \sqrt{S} \\ 0\end{array}\right) O^\top$ and $B = V \left(\begin{array}{c} \sqrt{S} \\ 0\end{array}\right) O^{\top}$. In particular, $\|AB^\top\|_*=\frac{1}{2}(\|A\|^2 + \|B\|^2)$.
\end{proposition}

This condition states that the scalar product of any two columns of $A$ should match the scalar product of corresponding columns of $B$. We will show next that at any stationary point of the objective $\mathcal{L}_{L2}$, that condition is fulfilled. We assume the loss $L$ is differentiable and $\lambda >0$.

\begin{lemma} \label{lemma:stationary}
At any stationary point $A,B$ of $\mathcal{L}_{L2}$ we have that $A^\top A = B^\top B$. 
\end{lemma}

The above Lemma, together with Proposition ~\ref{prop:norm_eq}, implies that at a stationary point $A,B$, $\mathcal{L}_{L2}(A,B)$ and $\mathcal{L}_*(AB^\top)$ coincide. However, this is not enough to claim that finding a (local) minimum of $\mathcal{L}_{L2}$ will in fact find a (local) minimum of $\mathcal{L}_*$. We now provide a result which shows that this claim is true.

\begin{theorem}\label{th:main}
$A,B$ is a local minimum of $\mathcal{L}_{L2}$, if and only if 1) $W=AB^\top$ is a local minimum of $\mathcal{L}_{*}$, constrained to matrices of rank $r$ where $r$ is the maximum rank achievable by $AB^\top$, and 2) $A^\top A=B^\top B$.
\end{theorem}

A more general formulation of the above results, albeit without a bottleneck dimension, was recently shown in \citep{wang2023implicit} (c.f. Theorem 3.1) where it was applied in the matrix completion context. We restate and prove it here for completion in the context of self-attention and transformer models. 

The Theorem states that there is in fact a one-to-one mapping between the local minima of $\mathcal{L}_{L2}(A,B)$ and (the equivalence class of) local minima of $\mathcal{L}_{*}(AB^\top)$, for a general unregularized loss $L$. 

In particular, if one wishes to optimize $\mathcal{L}_{*}$ for some matrix $W$, potentially under rank constraint, one can reparametrize $W$ as a product of two matrices $A,B$ and optimize the differentiable objective $\mathcal{L}_{L2}$ on $A,B$ without introducing bad minima, and obtain rank-regularized solutions. In principle, one can still converge to a bad minimum for a general loss, but this is not due to the reparametrization.

On the other hand, the theorem shows that naively optimizing the $L2$-regularized loss with a factorized parametrization will (often inadvertently) result in actually finding solutions that exactly minimize the nuclear-norm regularized loss, introducing unintended low-rank inducing bias to the solution.

Note that however, the two parametrizations may result in different optimization, and thus different solutions, even if the loss landscape shares the same local minima.

\subsection{Optimization dynamic in the gradient flow limit}\label{sec:l2_optimization}

The above result establishes equivalence of the local minima of the two losses. Our next result shows that the two losses will in fact coincide exponentially quickly during training.

\begin{theorem}\label{th:diff_loss}
Consider the gradient flow limit over the loss $\mathcal{L}_{L2}$. If $\|A\|,\|B\|$ remain bounded during training, then we have that $\left|\mathcal{L}_{L2}(A,B) - \mathcal{L}_*(AB^\top)\right| $ converges exponentially to $0$.
\end{theorem}

In order to prove the theorem, we first show that during gradient flow optimization, the condition from Proposition~\ref{prop:norm_eq} becomes true exponentially quickly. This is then followed by a new bound bounding the gap between $\|AB^\top\|_*$ and $\frac{1}{2}(\|A\|^2+\|B\|^2)$ by the norm of $A^\top A - B^\top B$. For completeness, we also provide a general result bounding the analogous gap for a $L$-layer deep linear network.

We provide in the appendix a similar result when considering gradient flow with noise, as well as with momentum and decoupled weight decay. We furthermore provide in appendix \ref{apx:discussion_boundness} a discussion about the soundness condition.

The above result complements Theorem \ref{th:main} by showing that optimizing $\mathcal{L}_{L2}$ will result in co-optimizing $\mathcal{L}_*$ very quickly during training, long before stationary points are found. The theorem also confirms previous empirical observations \citep{khodak2022initialization}.

\begin{figure*}[ht]
    \centering
    \hspace{-0.3cm}
    \begin{minipage}{.22\textwidth}
    \includegraphics[width=1\textwidth]{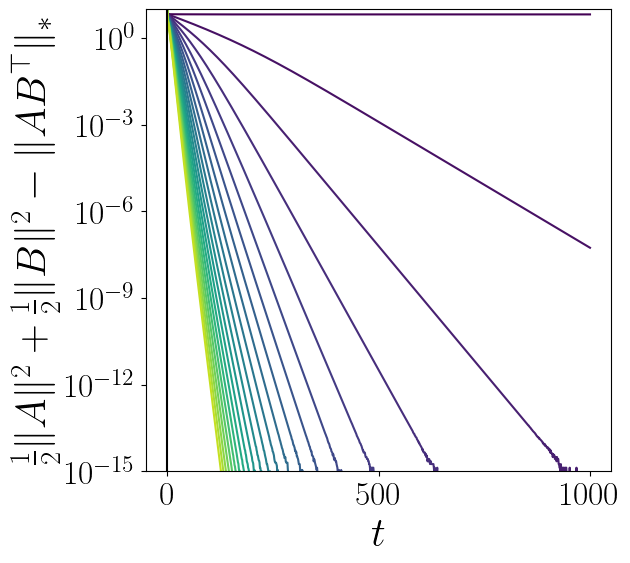}
    \end{minipage}
    \hspace{0.1cm}
    \begin{minipage}{.22\textwidth}
    \includegraphics[width=\textwidth]{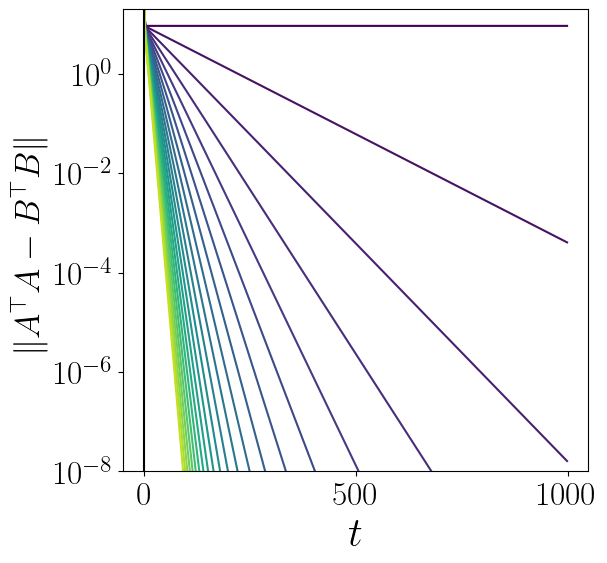}
    \end{minipage}
    \hspace{-0.8cm}
    \begin{minipage}{0.05\textwidth}
    \includegraphics[width=1\textwidth]
    {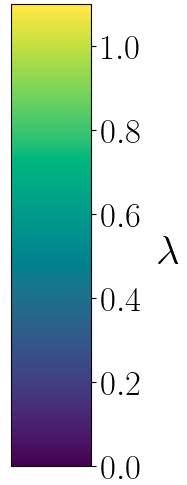}
    \vspace{0.9cm}
    \end{minipage}
    \hspace{0.1cm}
    \begin{minipage}{0.22\textwidth}
    \includegraphics[width=\textwidth]{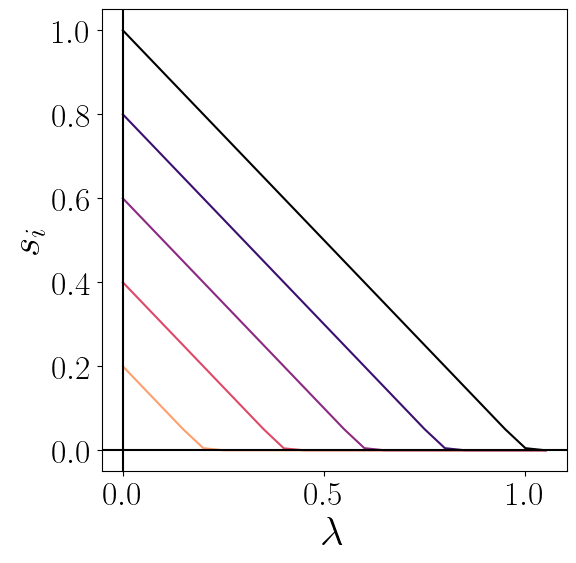}
    \end{minipage}
    \hspace{0.1cm}
    \begin{minipage}{0.22\textwidth}
    \includegraphics[width=\textwidth]{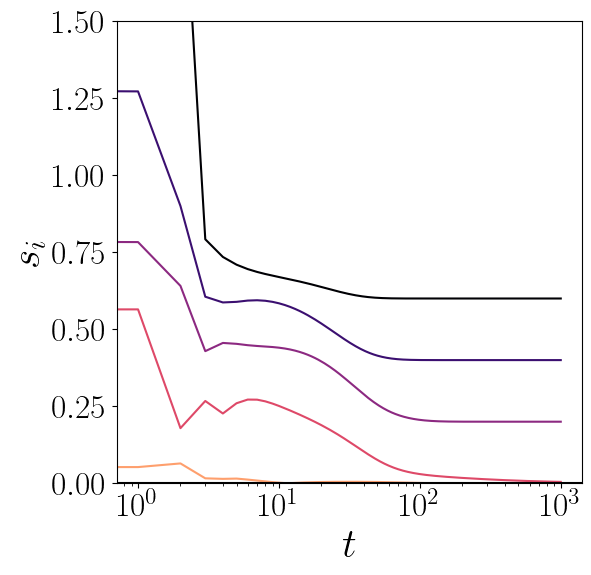}
    \end{minipage}
    \hspace{-1.1cm}
    \begin{minipage}{.07\textwidth}
    \includegraphics[width=1\textwidth]{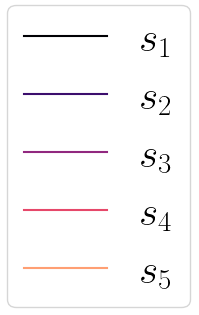}
    \vspace{1.4cm}
    \end{minipage}
    \caption{Optimization by gradient descent of two $5$-by-$5$ matrices $A,B$ on the $L2$-regularized loss $\|AB^\top - D\|^2 + \frac{\lambda}{2} (\|A\|^2 +\|B\|^2)$ where $D=\text{diag}(0.2,0.4,0.6,0.8,1)$, with various regularization strength $\lambda$. $t$ denotes the number of optimization steps. \textit{Left}: difference between the nuclear norm $\|AB^\top\|_*$ with the Frobenius norm $\frac{1}{2}\|A\|^2+\frac{1}{2}\|B\|^2$ throughout optimization. For all cases, other than $\lambda=0$, the trajectory converges exponentially quickly to $0$ as predicted by our theory. \textit{Center left}: Norm of the discrepancy between $A^\top A$ and $B^\top B$ over training steps. As predicted the discrepancy exponentially vanishes, with a time constant proportional to the $\lambda$.  \textit{Center right}: Singular values of the matrix $AB^\top$ at $t=1000$, for various regularization strength $\lambda$. As predicted, $s_i$ decays linearly with $\lambda$, until $\lambda \geq s_i$, at which point the singular value vanishes. \textit{Right}: Singular values of the matrix $AB^\top$ during optimization, for $\lambda = 0.4$. }
    \label{fig:toy_l2}
\end{figure*}

\subsection{Case study: 2-layer linear network}

To illustrate the low-rank inducing bias of the factorized parametrization coupled with weight decay, we will study in the following the optimization within a 2-layer linear network and characterize the network at equilibrium. Such a network corresponds in fact to a drastically simplified softmax attention layer with $T=1$. The derivations are similar to those used when studying deep linear networks \citep{ziyin2022exact, exact_slutions} and the redundant parameterization studied in \citep{pmlr-v202-ziyin23a}.

Consider the following model
\begin{equation}
f(AB^\top): x \rightarrow AB^\top x
\end{equation}
where $B^\top\in \mathbb{R}^{d_2 \times d_1}$, $A\in \mathbb{R}^{d_3 \times d_2}$. For simplicity of presentation, we assume $d_3 = d_1 = d_{1, 3}$, but the result can be easily extended to the general case.
Given $D$ data points $(x_i,y_i)_{1 \leq i \leq D}$, in matrix form, the $L2$-regularized mean squared error can be expressed as 
\begin{equation}
\label{2layer_reg_loss}
\mathcal{L} = \frac{1}{2} \|Y - AB^\top X\|^2 + \frac{\lambda}{2}(\|B\|^2+\|A\|^2)
\end{equation}
where $X=(x_i)_i\in \mathbb{R}^{d_1 \times D}$, $Y=(y_i)_i\in \mathbb{R}^{d_3 \times D}$, and $\lambda >0$.

Using full batch gradient flow, the differential equation governing the parameter dynamic becomes
\begin{align}
\tau \dot B^\top &= A^{\top}(\Sigma_{YX} - A B^\top \Sigma_{XX}) - \lambda B^\top \\
\tau \dot A &=  (\Sigma_{YX} - AB^\top \Sigma_{XX})  B - \lambda A
\end{align}
where $\Sigma_{YX}=YX^\top$, $\Sigma_{XX}=XX^\top$, and $\tau $ is a constant controlling the learning rate.

To further simplify the above equations, we follow \citep{exact_slutions} and assume $\Sigma_{XX}=I$, an assumption which holds exactly for whitened input data. Finally, without loss of generality, we perform a change of basis such that $\Sigma_{YX}=S$ where $S$ is the diagonal matrix which diagonal consists of the singular values $(s_i)_{i\in [1..d_{1, 3}]}$ of $YX^\top$. 

At equilibrium, we thus have the following set of equations
\begin{align}
\lambda B^\top &= A^\top (S - A B^\top)   \\
\lambda A &=  (S - AB^\top ) B.
\end{align}

Denoting by $a_i, b_i$ the $i$-th row of $A,B$, and assuming the $(s_i)_{i\in [1..d_{1, 3}]}$ are all  non-zero and distinct, we have the following conditions at equilibrium (cf Appendix~\ref{app:deeplinear})
\begin{align}
\forall i \in [1..d_{1, 3}], \text{ } &a_i = b_i  \\
\forall i, j \in [1..d_{1, 3}]^2 \text{ s.t. } i \neq j, \text{ } &a_i^\top b_j = 0.
\end{align}

In particular, this implies, for any $i$, $\lambda \|a_i\|^2 = (s_i - \|a_i\|^2) \|a_i\|^2$.
Clearly, if $\lambda \geq s_i$, then the equation can only be true if $a_i=0$. If on the other hand $\lambda < s_i$, either $a_i=0$ or $\|a_i\|^2 = s_i - \lambda$ satisfy the equilibrium condition, with the former being an unstable equilibrium point if the number of hidden units $d_2$ is greater than the number of elements in $\{i \in [1..d_{1, 3}]\mid s_i > \lambda \}$.

To highlight the result, let us consider the case where the hidden layer has enough capacity, i.e. $d_2\geq d_{1, 3}$. In that case, the result tells us that at a stable equilibrium, $AB^\top$ will drop all singular values $s$ that are less than $\lambda$, while keeping those that are larger. In other words, it performs a sort of low rank approximation of the input-output correlation matrix where the rank is controlled by $\lambda$. A related result was already obtained in the analyses of \citep{exact_slutions} who studied the exact solutions of \ref{2layer_reg_loss} without the regularization term but introducing a bottleneck in the hidden layer, i.e. $d_2 < d_{1, 3}$. Remarkably here, regularization achieves a similar effect even in an overcomplete network, where increasing $\lambda$ gradually \textit{prunes} the hidden neurons to ignore the smallest variations of the data, i.e. reducing $d_2$ adaptively. We confirm these results empirically in Figure~\ref{fig:toy_l2}.

Importantly, this result is only obtained because the regularization is applied to the parametrization involving a matrix multiplication. If $AB^\top$ were replaced by a single matrix $W \in \mathbb{R}^{d_{1,3}\times d_{1,3}}$, then the equilibrium condition would be $W = \frac{1}{1+\lambda} S$, whose rank remains constant w.r.t. the regularization strength.

\subsection{Weight decay with Adam optimizer}

While the regularized loss is a convenient setting for studying what happens to the parameters at equilibrium, in the vast majority of practical settings, decoupled weight decay \citep{loshchilov2018decoupled}, simply referred to as weight decay in the following, is used instead optimizing a regularized loss. A popular choice of optimizer for deep neural networks, including those with self-attention layers, is AdamW \citep{loshchilov2018decoupled}, which update the weights by using the Adam optimizer on the non-regularized loss while simultaneously applying weight decay. 

While it is non-trivial to analyze the equilibrium points of AdamW in general, we show in Appendix~\ref{app:adamw} that under some simplifying assumptions, they coincide with those of a $L2$-regularized loss with a different regularization strength.

\section{Empirical results}

\begin{figure*}[ht] \label{fig:icl}
    \centering
    \hspace{-0.1cm}
    \begin{minipage}{.3\textwidth}
       \includegraphics[width=1 \textwidth]{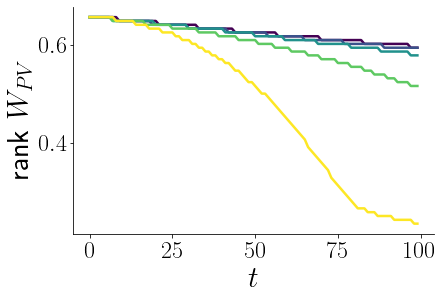}
    \end{minipage}    
    \begin{minipage}{.3\textwidth}
       \includegraphics[width=1 \textwidth]{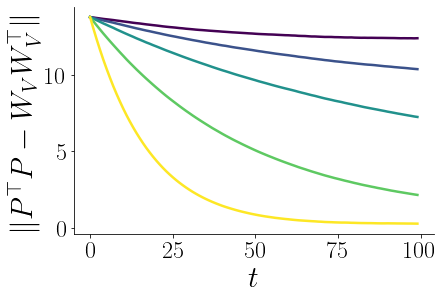}
    \end{minipage}
    \begin{minipage}{.3\textwidth}
        \includegraphics[width=1 \textwidth]{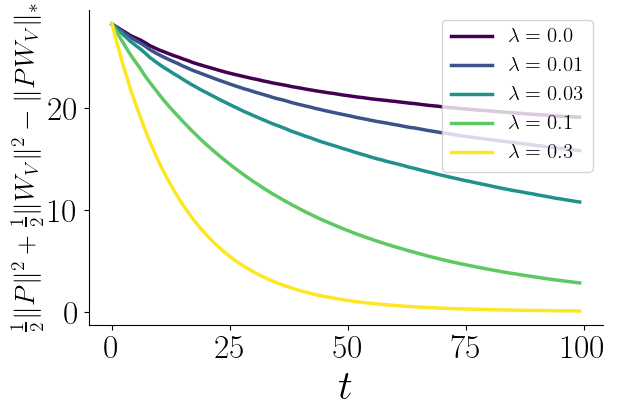}
    \end{minipage}
\caption{
\textit{Left}: The rank of weight matrix product $PW_V$  of the first layer of a 2-layer Transformer trained on the associative recall task, during training, with AdamW, for various decay strengths. To better account for the effect of weight decay on the attention layers, only the decay strength applied to attention layers is varied, while the strength for all other layers is fixed at $0.1$. We observe that rank reduction correlates strongly with weight decay strength. 
\textit{Center}: Norm of the discrepancy between $P^\top P$ and $W_VW_V^\top$, during training. As predicted, the difference seems to converge to $0$ when $\lambda >0$ towards the end of training. While for AdamW we no longer have the guarantee of an exponential decay, we see that the discrepancy nonetheless vanishes quickly, with a time constant which perfectly correlates with the decay strength. 
\textit{Right}: The difference of the nuclear norm of $W_{VP}$ with the Frobenius norm upper bounding it. As the discrepancy between $P^\top P$ and $W_VW_V^\top$ decreases, the difference approaches $0$, and thus the bound becomes tight. The optimization of $\mathcal{L}_{L2}$ thus gradually switches to that of $\mathcal{L}_*$, explaining the rank regularization. Qualitative findings are identical when studying $W_K^\top W_Q$.}
\end{figure*}

The primary objective of our experimental analysis is to empirically validate the theoretical findings in more practical settings. Specifically, we aim to investigate the effect of decoupled weight decay, adaptive optimizers, as well as noisy gradient and lack of exact convergence to stationary points on the theoretical findings.

The second objective is to establish that the theory is relevant in the training of large neural network models. Due to the large computational costs we chose to avoid re-training large scale models but trained small-scale language models as well as a Vision Transformer without changing common hyperparameters. We aim to demonstrate that their typical training is affected by the rank-regularizing effect predicted by our theory. Finally, we investigate pre-trained weights of the relevant foundation models to show that they are consistent with rank-regularizing training.

To quantitatively measure the rank of matrices in the context of our experiments with attention layers, we use the following definition of \textit{pseudo rank}: Let \(W\) be a weight matrix with singular values \(\sigma_1, \sigma_2, ..., \sigma_n\), ordered such that \(\sigma_1 \geq \sigma_2 \geq ... \geq \sigma_n\). The pseudo rank (referred to simply as rank in the following) of \(W\) is defined as $\frac{k}{n}$ where $k$ is the smallest number such that:
\[
\frac{\sum_{i=1}^{k} \sigma_i}{\sum_{i=1}^{n} \sigma_i} \geq 0.95.
\]
In simpler terms, it represents the fraction of the largest singular values required to capture at least 95\% of the sum of all singular values of the matrix \(W\). 

\begin{figure*}[ht]
    \begin{minipage}{.245\textwidth}\includegraphics[width=1.\textwidth]{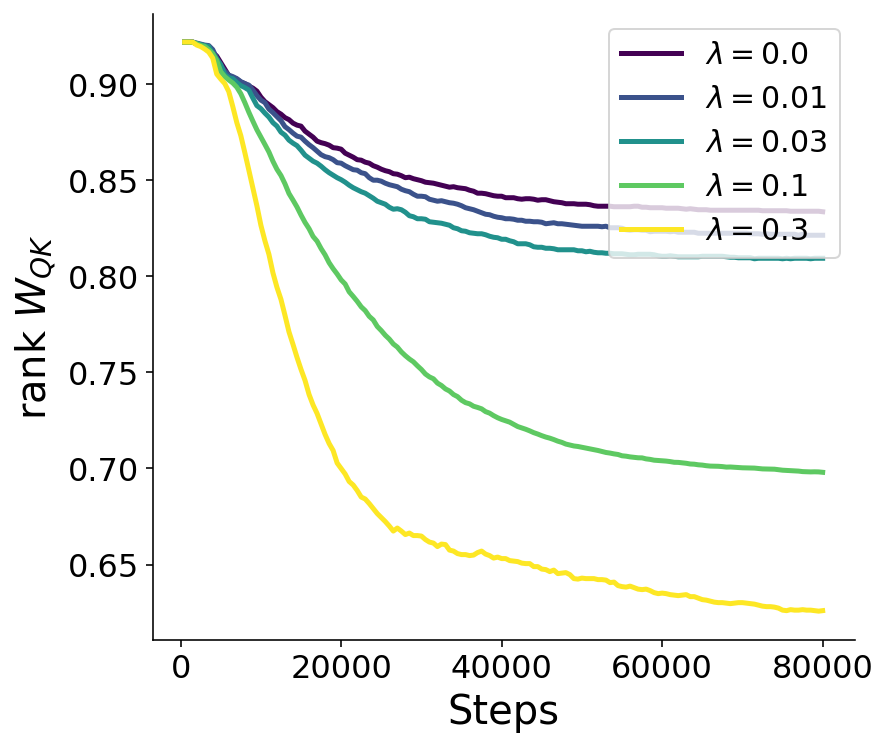}
    \end{minipage}
    \begin{minipage}{.245\textwidth}
    \includegraphics[width=1.\textwidth]{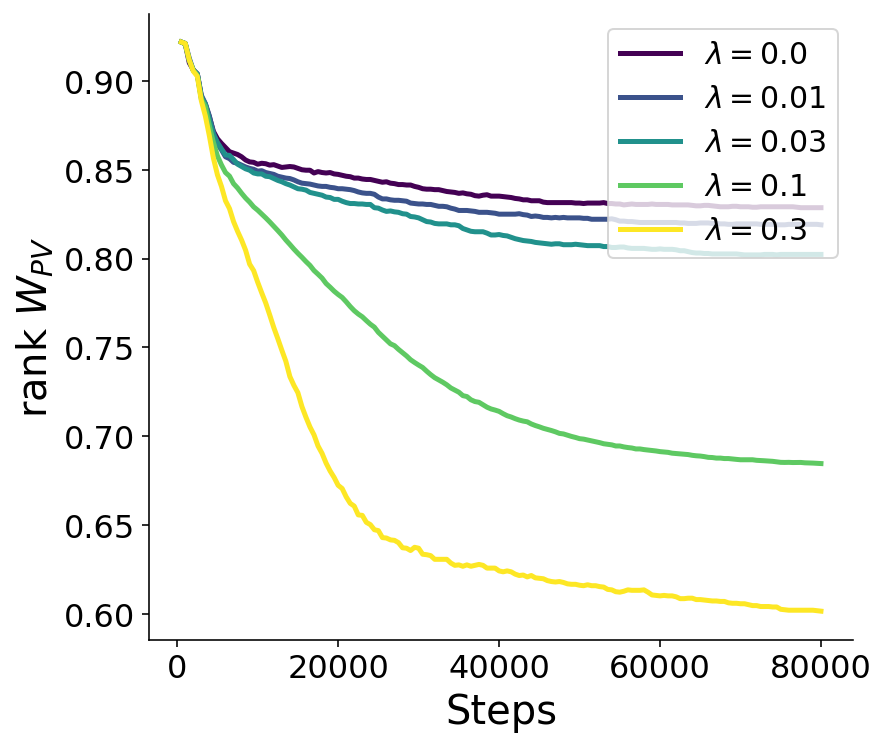}
    \end{minipage}
    \begin{minipage}{.245\textwidth}\includegraphics[width=1.\textwidth]{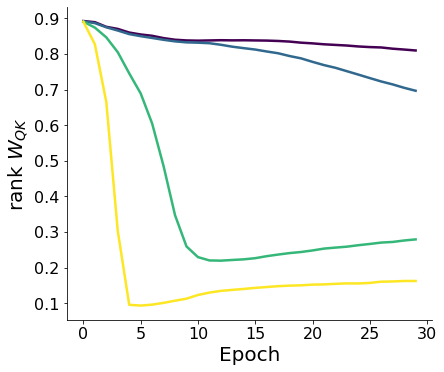}
    \end{minipage}
    \begin{minipage}{.245\textwidth}
    \includegraphics[width=1.\textwidth]{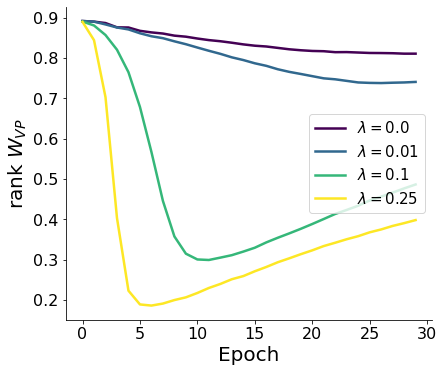}
    \end{minipage}
    \caption{\textit{Left, center left}: The rank of weight matrix products $W_K^\top W_Q$ and $PW_V$  averaged across heads of layer 5 of an autoregressive transformers trained on the Pile \citep{gao_pile_2020}. 
    \textit{Center right, right}: The rank of weight matrix products $W_K^\top W_Q$ and $PW_V$ averaged over all heads and all layers of a Vision Transformer trained following \citep{irandoust2022training} on the ImageNet dataset \citep{deng2009imagenet}.
    In both settings, the decay strength applied to attention layers is varied, while keeping the strength for all other layers fixed. In all cases, we observe again that rank reduction correlates strongly with weight decay strength when optimizing with AdamW. The weight decay strength of $0.1$ commonly used to pretrain some known large foundation models in fact noticeably reduces the rank of the generated matrices compared to when weight decay is turned off. 
    }
    \label{fig:llms}
\end{figure*}






\subsection{Associative recall task}

In this simple memory task, a model is presented with a sequence of paired tokens $[x_1, y_1, \dots, x_T, y_T, x_{T+1}]$. Specifically, the task is parameterized by an integer \(N\), representing the number of unique tokens that can be mapped to \(N\) corresponding tokens. The sequence presented to the model therefore consists of \(2N + 1\) tokens (with $T=N$), and the final token is repeated and appears in the sequence before, i.e. $x_{T+1} = x_j$ for some $j \in [0, \dots, T]$. The model is trained to remember the correct association observed in-context and predict $y_j$. This task has been attributed and proposed as a proxy for language modelling \citep{fu2023hungry, poli2023hyena}. 

We train a 2-layer self-attention only Transformer with AdamW optimizer on minibatches of size 128, for $N=20$. To simulate additional noise, we perturb $5\%$ of the labels with random labelling \citep{10.1145/3446776}. 

Figure~\ref{fig:icl} shows that even in this setting, the stationary condition of a $L2$-regularized loss in Lemma~\ref{lemma:stationary} is approached, and the gap between the nuclear norm and the Frobenius norm in \eqref{eq:gap} vanishes, thus confirming that AdamW in fact also optimizes for the nuclear norm. Furthermore, the convergence speed is perfectly correlated with the weight decay strength. The results furthermore show that AdamW leads indeed to a consistent decrease in the rank in both parameter weight products as the decay strength increases. This aligns with the effect of optimizing the nuclear norm of these matrices. 

\subsection{Language Modelling}

In order to validate our theoretical findings in larger scale experiments, we now present results when training standard small scale Transformer models, with 125 million parameters, on the Pile \citep{gao_pile_2020} - a common language modeling dataset. All design decisions such as the Transformer architecture as well as the optimizer and training schedule are identical to the ones proposed in the GPT-3 paper \citep{transformers_few_shot}, which are now used in various other studies e.g. \citep{fu2023hungry, vonoswald2023uncovering}. Details can be found in the Appendix \ref{app:llm}.

\begin{table}
\centering
\caption{Test set perplexity of 125 million Transformer models trained on the Pile for 10 billion tokens with AdamW and different weight decays $\lambda$ for the self-attention (SA) and the feed-forward (MLP) weights. $\pm$ standard error of the mean computed over 5 seeds.}
  \label{tab:results_table}
   \begin{tabular}{lccccc}

    \toprule
    & SA-$\lambda=0.0$ & SA-$\lambda=0.01$ &SA-$\lambda=0.025$   &SA-$\lambda=0.1$ &  SA-$\lambda=0.25$ \\
    \midrule
MLP-$\lambda=0.0$    & 12.00$^{\pm0.03}$& 12.01 $^{\pm0.05}$& 11.98 $^{\pm0.00}$& 11.92 $^{\pm0.02}$& 12.02 $^{\pm0.01}$\\
MLP-$\lambda=0.01$  & 11.94$^{\pm0.02}$ & 11.95$^{\pm0.01}$ & 11.94$^{\pm0.03}$ & 11.89$^{\pm0.02}$ & 11.97$^{\pm0.03}$\\
MLP-$\lambda=0.025$ & 11.89$^{\pm0.01}$ & 11.90$^{\pm0.04}$ & 11.90$^{\pm0.04}$ & 11.80$^{\pm0.03}$ & 11.92$^{\pm0.02}$ \\
MLP-$\lambda=0.1$   & 11.72$^{\pm0.02}$ & 11.71$^{\pm0.03}$ & 11.68$^{\pm0.03}$ & 11.67$^{\pm0.03}$ & 11.70$^{\pm0.02}$ \\
MLP-$\lambda=0.25$  & 11.63$^{\pm0.02}$ & 11.65$^{\pm0.04}$ & 11.62$^{\pm0.04}$ & 11.52$^{\pm0.03}$ & 11.58$^{\pm0.03}$ \\
    \bottomrule
  \end{tabular}
  \label{tab:llm}
\end{table}

First, we confirm again that increasing weight decay with AdamW drastically reduces the rank of $W_K^TW_Q$ as well as $PW_V$, on average across depth and heads, of the trained models (c.f. Figure \ref{fig:llms}. Appendix Figure \ref{fig:llms_app}). Furthermore, we observe that while increasing the weight decay strength of MLP beyond 0.1 is generally beneficial, see Table \ref{tab:llm}, doing the same for attention matrices starts slightly hurting performance. Results are averaged over 3 seeds. We observe a sweet spot around weight decay strength of 0.1 applied to self-attention weights, indicating that some rank regularization is beneficial for this task. Nevertheless, too much weight decay and therefore rank regularization seems to be detrimental. Finally, applying weight decay to the MLP weights seems to more important with a generally higher effect on performance. We leave a more nuanced investigation of decoupling the weight decay strength of matrices affected by our theory from the rest of the parameters for future research.

\subsection{Vision Transformers}
Next, we focus on computer vision tasks and train a Vision Transformer on the ImageNet dataset \citep{deng2009imagenet} for 24 hours, following the exact training protocol of \citep{irandoust2022training}. We follow the previous section and vary the decay strength only in the attention layers, while keeping every other hyperparameter fixed. We observe a similar effect of the decay strength on the ranks of the matrices $W_{QK}, W_{VP}$ (c.f. Fig~\ref{fig:llms}).

\subsection{Pretrained foundation models}

Finally, we turn to pre-trained foundation models, and provide some evidence that their training is also impacted by the rank-regularizing effect of weight decay. Specifically, following Proposition~\ref{prop:diff_bound}, it is sufficient to observe that the matrices $W_Q W_Q^\top$ resp. $P^\top P$ are close to $W_K W_K^\top$ resp. $W_V W_V^\top$. Because the matrices $W_Q, W_K, W_V, P^\top$ are typically wide rectangular matrices, the off-diagonal elements of $W_Q W_Q^\top$, etc, are mostly $0$. For $\mathcal{L}_{L2}$ to approximately correspond to $\mathcal{L}_*$, it thus suffices that the diagonal elements of $W_Q W_Q^\top$ resp. $P^\top P$ are close to those of $W_K W_K^\top$ resp. $W_V W_V^\top$.

Figure~\ref{fig:pretrained}, ~\ref{fig:pretrained_vit} shows that this is mostly the case, for all layers of the model. For each layer and head, we further find that the gap from \eqref{eq:gap} is indeed mostly tight, consistent with a rank regularizing training.

\begin{figure*}[ht]
    \centering
    \hspace{-0.1cm}
    \begin{minipage}{.2\textwidth}
        \includegraphics[width=1 \textwidth]{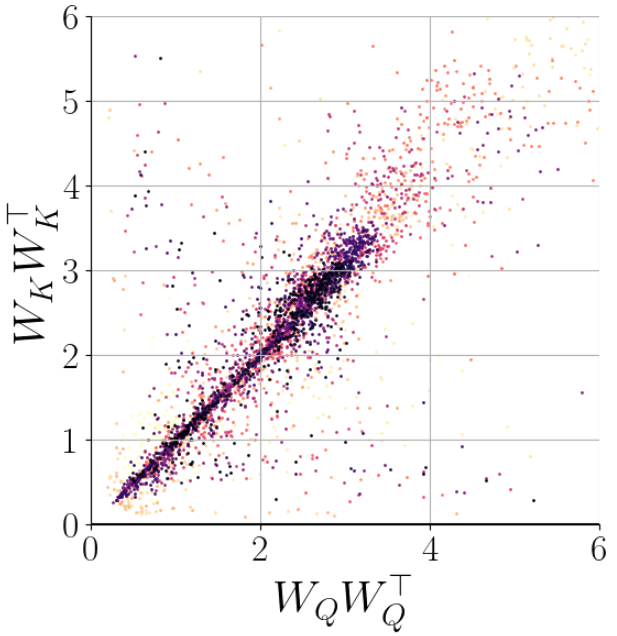}
    \end{minipage}
    \begin{minipage}{.2\textwidth}
       \includegraphics[width=1 \textwidth]{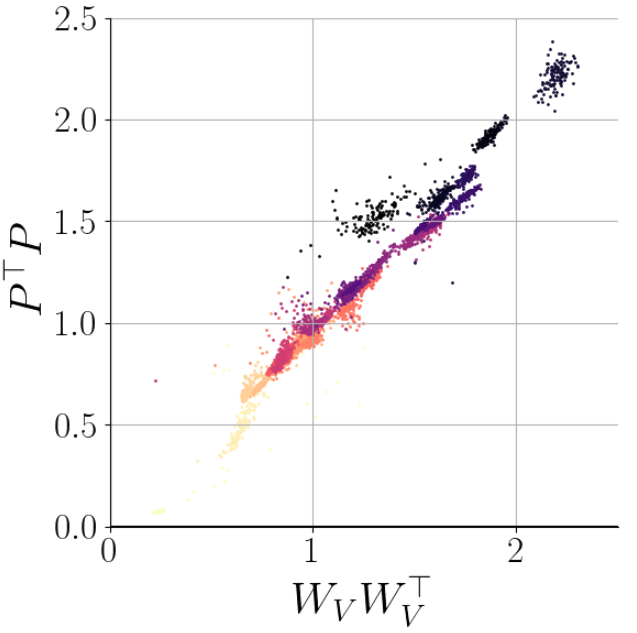}
    \end{minipage}
    \begin{minipage}{.2\textwidth}
        \includegraphics[width=1 \textwidth]{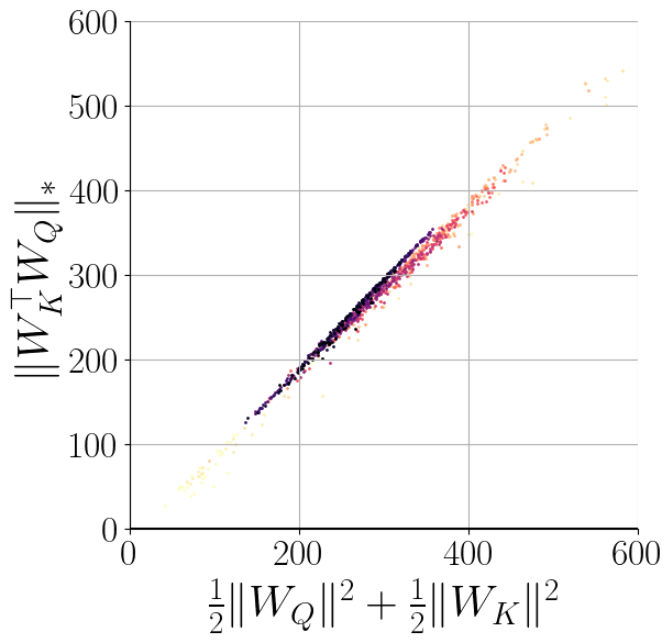}
    \end{minipage}
    \begin{minipage}{.2\textwidth}
       \includegraphics[width=1 \textwidth]{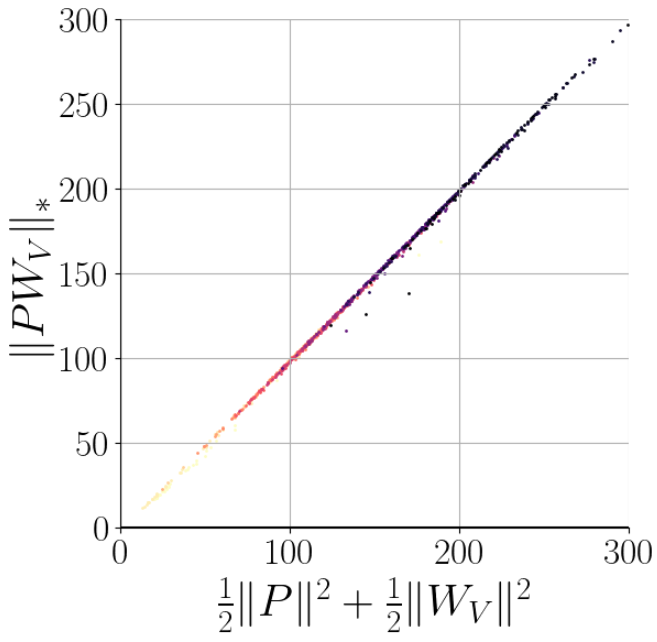}
    \end{minipage}
    \begin{minipage}{.1\textwidth}
    \hspace{0.05cm}
       \includegraphics[width=0.7 \textwidth]{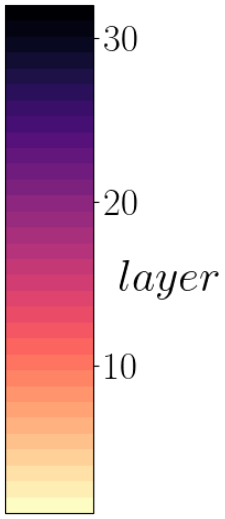}
       \vspace{0.55cm}
    \end{minipage}
    \caption{ Analyses of attention layers in the pretrained LLAMA 2 model with 7 Billion parameters \citep{touvron2023llama}. The leftmost (resp. center left) shows the squared norm of every row of $W_Q$ (resp. $W_V$), for the first head of each layer, against the norm of the corresponding row of $W_K$ (resp. column of $P$). The condition $W_KW_K^\top = W_QW_Q^\top$ would require these norms to be equal, which in fact is mostly true. While the model has not reached a stationary point, this indicates the optimization has advanced enough for this sufficient condition for $\mathcal{L}_*$ to be identical to $\mathcal{L}_{L2}$ to emerge. In fact, the center right (resp. rightmost) plot show the scatter plot mapping the Frobenius norm against the nuclear norm for all heads across all layers. The two norms almost perfectly coincide.} 
     \label{fig:pretrained}
\end{figure*}
\section{Discussion}
Our results provide further insights into the interplay between $L2$-regularization and weight decay regularization and the optimization of models that consist of parameter matrix products. This is of particular interest since attention layers in transformer exhibit this parametrization as key-query, as well as value-projection parameter matrices, are multiplied directly with each other: $W_K^TW_Q$ and $PW_V$. Our empirical findings strongly support our theoretical predictions about the impact of weight decay on the rank of attention layers and clearly show a rank-regularizing effect even without convergence. We provide evidence that the training of some foundation models such as Llama are in fact in practice affected by the same regularization.

Furthermore, we find that decoupling weight decay in the attention weights and tuning its weight decay strength can improve performance, for example in our language modelling experiments. These findings complement the recent observation that reducing the rank of language model MLP matrices post-training improves their reasoning performance, while doing the same for attention layer matrices mostly hurt it \citep{sharma2023truth}. In particular, our findings suggest that the conventional practice of applying uniform regularization strategies across all layers may not be optimal for other deep learning architectures as well. This finding opens up new avenues for model- or layer-specific regularization strategies that could significantly enhance the performance of these models.

Our findings once more highlight the complexity of understanding optimization techniques in conjunction with particular neural network models, particularly transformers. For example, the difficulty of understanding the effect when varying regularization strengths on different components of these models underscores the need for a more nuanced theoretical understanding of layer-specific regularization. We are particularly excited about further research that aims to disentangle the role of weight decay in in-weight vs. in-context learning within MLPs and self-attention layers, building on \citep{singh2023transient}.
In conclusion, while our findings mark a step forward in understanding and improving the usage of weight decay when training deep neural networks, in particular transformers, our study shed light on the intricate interplay of neural network regularization and its parametrization.

\subsection*{Acknowledgments}
Seijin Kobayashi, Yassir Akram and Johannes von Oswald deeply thank João Sacramento and Angelika Steger for their support and guidance. The authors also thank Robert Obryk, Moritz Firsching, Luca Versari, Nicolas Zucchet, Alexander Meulemans, Simon Schug, Blaise Agüera y Arcas, Alexander Mordvintsev, Ettore Randazzo and Eyvind Niklasson for fruitful discussions. 

\newpage

\bibliography{refs}
\bibliographystyle{unsrtnat}

\newpage
\appendix
\onecolumn
\section{Compute budget}
We estimate the total compute budget to 4 Nvidia RTX 4090 for two months. The LLMs were punctually trained on a cluster of 16xA100 GPUs for 4 days.  

\section{Proofs of theoretical results} \label{app:proof}
\subsection{Proof of Proposition \ref{prop:norm_eq}} \label{app:proof_norm_eq}

\begin{proof}
Using singular value decomposition, we write $A=U_A\Sigma_AV_A^\top$ and $B=U_B\Sigma_BV_B^\top$, with $\Sigma_A = \left(\begin{array}{c} S_A \\ 0\end{array}\right)$ and $\Sigma_B = \left(\begin{array}{c} S_B \\ 0\end{array}\right)$. Substituting in the equation  $A^\top A=B^\top B$, we get 

$$V_A S_A^2 V_A^\top = V_B S_B^2 V_B^\top$$

By positivity and uniqueness of singular values, necessarily, $\Sigma_A = \Sigma_B = \left(\begin{array}{c} S \\ 0\end{array}\right)$. Furthermore, by rearranging the above equation, we get $S^2 V_A^\top V_B = V_A^\top V_B S^2 $, i.e. that $V_A^\top V_B$ commutes with $S$.
We rewrite $A$ as
\begin{align*}
    A &= U_A \Sigma_A V_A^{\top} V_B V_B^{\top} = U_A \left(\begin{array}{c} S V_A^{\top} V_B \\ 0\end{array}\right) V_B^{\top} \\
    &= U_A \left(\begin{array}{c} V_A^{\top} V_B S \\ 0\end{array}\right) V_B^{\top} = U_A \left(\begin{array}{cc} V_A^{\top} V_B & 0 \\ 0 & I \end{array}\right) \Sigma_A V_B^{\top}
\end{align*}

Redefining $U_A$ as $U_A \left(\begin{array}{cc} V_A^{\top} V_B & 0 \\ 0 & I \end{array}\right)$, setting $\Sigma = \Sigma_A$, and $V = V_B$, we can write $A = U_A \Sigma V^\top$ and $B = U_B \Sigma V^{\top}$. 

In particular, $AB^\top = U_A \Sigma \Sigma^{\top} U_B^\top$, which is a valid SVD of $AB^\top$.
It remains to show that, if $AB^\top$ is diagonal, then there exists an orthogonal matrix $O$ such that $A = \Sigma O^\top$ and $B = \Sigma O^{\top}$.

Let us assume the diagonality, i.e. $AB^\top = \Sigma \Sigma^{\top}$. Then, we have 
\begin{align*}
    (AB^\top)^2 = U_A \Sigma \Sigma^{\top}\Sigma \Sigma^{\top}U_A ^\top = \Sigma \Sigma^{\top}\Sigma \Sigma^{\top} = U_B\Sigma \Sigma^{\top}\Sigma \Sigma^{\top}U_B^\top
\end{align*}
i.e. that $U_A,U_B$ commute with $\Sigma \Sigma^{\top}$, and thus that they are block diagonal. Furthermore, $\Sigma \Sigma^{\top} U_A  U_B^\top = \Sigma \Sigma^{\top}$. They can then be written as 
$$U_A = \left(\begin{array}{cc} U & 0 \\ 0 & U_A' \end{array}\right) $$
$$U_B = \left(\begin{array}{cc} U & 0 \\ 0 & U_B' \end{array}\right) $$
where $U, U_A', U_B'$ are orthogonal matrices, and the block of $U$ corresponds to the non zero singular values of $\Sigma \Sigma^{\top}$.

We can then rewrite $A,B$ as $A = \Sigma \left(\begin{array}{cc} U & 0 \\ 0 & I \end{array}\right) V^\top$ and $B = \Sigma \left(\begin{array}{cc} U & 0 \\ 0 & I \end{array}\right) V^{\top}$, which conclude the proof by setting $O=\left(\begin{array}{cc} U & 0 \\ 0 & I \end{array}\right) V^{\top}$.

Finally, $AB^\top = U_A \Sigma \Sigma^{\top} U_B^\top$, and therefore $\|AB^\top\|_* = \|U_A \Sigma \Sigma^{\top} U_B^\top\|_* = \Tr(\Sigma \Sigma^{\top}) = \frac{1}{2}(\|A\|^2 + \|B\|^2)$.
\end{proof}

\subsection{Proof of Lemma \ref{lemma:stationary}} \label{app:proof_stationary}

\begin{proof}
Let $A,B$ a stationary point of the unregularized loss $L$ in $\mathcal{L}_{L2}$.
One can show that the gradient of $L(W=AB^\top)$ with respect to $A$ (resp. $B$) is
\begin{align}
    \partial_A L &= \left(\frac{\partial L}{\partial W} \big |_{W=AB^\top}\right) B \\
    \partial_B L &= \left(\frac{\partial L}{\partial W} \big |_{W=AB^\top}\right)^\top A 
\end{align}
where $\frac{\partial L}{\partial W} \big |_{W=AB^\top}$ is a matrix, which we denote by $-G$. Differentiating $L$, at the stationary point, the following equations must then be satisfied
\begin{align}
    \lambda A &= G B \\
    \lambda B &= G^\top A 
\end{align}
In particular, $A^\top A = \frac{1}{\lambda} A^\top GB = \frac{1}{\lambda}(G^\top A)^\top B=B^\top B$.

\end{proof}

\subsection{Proof of Theorem \ref{th:main}} \label{app:proof_theorem}

\begin{proof}
($\Leftarrow$)    We start by proving the backward implication, by contradiction. Let $M$ a local minimum of $\mathcal{L}_{*}$, and $A,B$ such that $M=AB^\top$ and $A^\top A=B^\top B$. Then by Proposition~\ref{prop:norm_eq}, $\mathcal{L}_{L2}(A,B) = \mathcal{L}_{*}(M)$. Assume $A,B$ is not a local minimum of $\mathcal{L}_{L2}$, i.e. there exists an infinitesimally perturbed matrices $A', B'$ such that $\mathcal{L}_{L2}(A', B') < \mathcal{L}_{L2}(A,B)$. By continuity of matrix multiplication, $M'=A'B'^\top$ is an infinitesimally perturbed matrix $M$. Since $\mathcal{L}_{*}(M') \leq \mathcal{L}_{L2}(A', B') < \mathcal{L}_{L2}(A,B) = \mathcal{L}_{*}(M)$, we get a contradiction.
    
($\Rightarrow$)    Assume now that $A,B$ is a local minimum of $\mathcal{L}_{L2}$, and that $W=AB^\top$ is not a local minimum of $\mathcal{L}_{*}$ constrained to rank $r$ matrices. Then, we can construct a sequence $(W_n)_n$ of rank $r$ matrices such that $\lim_{n \to \infty} W_n=W$, and for all $n$, $\mathcal{L}_{*}(W_n)<\mathcal{L}_{*}(W)$. For all $n$, let $W_n = U_n S_n V_n^\top$ the SVD of $W_n$. By continuity of the mapping from a matrix to its singular values, $\lim_{n \to \infty} S_n=S$, where $S$ are the singular values of $W$. Because the set of orthogonal matrices is compact, there exists a subsequence of $((U_n, V_n))_n$ which converges to some orthogonal matrices $(U,V)$. Without loss of generality, we redefine the sequence to this converging subsequence. By continuity of matrix multiplication, necessarily $USV^{\top} = W$. $USV^{\top}$ is a valid SVD of $W$. Since by local minimality of $A,B$, following Lemma~\ref{lemma:stationary} and Proposition~\ref{prop:norm_eq}, we get that $A=U \Sigma O^\top$ and $B=V \Sigma O^\top$ where $\Sigma = \left(\begin{array}{c} \sqrt{S} \\ 0\end{array}\right)$ and $O$ is some orthogonal matrix. Let for all $n$, $A_n=U_n \Sigma_n O^\top$ and $B_n=V_n \Sigma_n O^\top$, where $\Sigma_n = \left(\begin{array}{c} \sqrt{S_n} \\ 0\end{array}\right)$. Then, $\lim_{n \to \infty} (A_n,B_n) = (A,B)$ and yet, because for all $n$, $A_n^\top A_n = B_n^\top B_n$ and $A_n B_n^\top = W_n$, we have $\mathcal{L}_{L2}(A_n,B_n) = \mathcal{L}_{*}(W_n) < \mathcal{L}_{*}(W) = \mathcal{L}_{L2}(A,B)$. This is a contradiction.
\end{proof}

\subsection{Proof of Theorem \ref{th:diff_loss}} \label{app:proof_diff_loss}

In order to prove the theorem, we first show that during optimization, the condition from Proposition~\ref{prop:norm_eq} becomes true exponentially quickly. This is then followed by a new bound bounding the gap between $\|AB^\top\|_*$ and $\frac{1}{2}(\|A\|^2+\|B\|^2)$ by the norm of $A^\top A - B^\top B$.

\subsubsection{Exponential decay of \texorpdfstring{$A^\top A - B^\top B$}{ATA-BTB}}
\label{proof:expdecay}

We provide the result for the vanilla gradient flow limit, but also provide an alternative proof for the stochastic gradient flow with momentum and decoupled weight decay, to illustrate that the exponential decay would hold in many practical settings. We note that the gradient flow limit is a good approximation for a small learning rate in the discrete dynamic.

\begin{lemma} \label{lemma:exp_decay}
In the gradient flow limit over the loss $\mathcal{L}_{L2}$, $A^\top A - B^\top B$ will converge exponentially to $0$.
\end{lemma}

\begin{proof}
For any $i$, we denote by $a^i,b^i$ the $i$-th column of $A$ and $B$. The columns follow the following differential equations:
\begin{align}
    \tau \dot a^i &= G b^i - \lambda a^i \\
    \tau \dot b^i &= G^\top a^i - \lambda b^i 
\end{align}
where $G=-\frac{\partial L}{\partial W} \big |_{W=AB^\top}$, and $\tau $ is some time constant controlling the learning rate. Given a pair $i,j$, we can now look at the dynamic of $a^{i\top} a^j-b^{i\top} b^j$:
\begin{align*}
\tau \frac{d}{dt}(a^{i\top} a^j-b^{i\top} b^j) &= \tau(  a^{j\top} \dot a^i + a^{i\top} \dot a^j -  \dot b^{j\top}b^i - \dot b^{i\top} b^j) \\
&= a^{j\top} G b^i - \lambda a^{j\top} a^i  \\
&+ a^{i\top} G b^j -  \lambda a^{i\top} a^j \\
&- (a^{j\top} G b^i -  \lambda b^{j\top} b^i)  \\
&- (a^{i\top} G b^j - \lambda b^{i\top} b^j)  \\
&= -2\lambda (a^{i\top} a^j-b^{i\top} b^j)
\end{align*}
Therefore, we have $A^\top A - B^\top B = Q e^{-\frac{2\lambda t}{\tau}}$, where $Q$ is $A^\top A - B^\top B$ at initialization, and in particular, every entry of $A^\top A - B^\top B$ converge to $0$ exponentially.
\end{proof}

We now provide a similar result, in the gradient flow regime but with momentum, as well as decoupled weight decay - a tractable approximation to AdamW as shows the following proposition:

\begin{proposition} \label{prop:exp_decay}
We consider the following dynamics approximating stochastic gradient flow with weight decay:
\begin{align*}
    d H^A_t &= \mu (G_t B_t dt + \sigma dW^A_t - H^A_t dt)\\
    d H^B_t &= \mu (G_t^\top A_t dt + \sigma dW^B_t - H^B_t dt)\\
    dA_t &= - \eta (H^A_t+\lambda A_t)dt\\
    dB_t &= - \eta (H^B_t+\lambda B_t)dt
\end{align*}
where $\mu, \eta, \sigma > 0$ and $W^A$ and $W^B$ are independent matrix Wiener processes. Initial condition are $H^A_0 = 0$ and $H^B_0 = 0$.
$H^A$ (resp. $H^B$) is the momentum gradient with respect to $A$ (resp. $B$).
Then, \begin{align}
    H^A_t =& \mu \int_0^t e^{-\mu (t-s)} G_s B_s ds + \sqrt{\frac{\mu \sigma^2}{2}} W^A_{1-e^{-2\mu t}}\\
    H^B_t =& \mu \int_0^t e^{-\mu (t-s)} G^\top_s A_s ds + \sqrt{\frac{\mu \sigma^2}{2}} W^B_{1-e^{-2\mu t}}\\
    A_t =& e^{-\eta\lambda t} A_0 - \eta\int_0^t e^{-\eta\lambda(t-s)} H^A_s ds\\
    B_t =& e^{-\eta\lambda t} B_0 - \eta\int_0^t e^{-\eta\lambda(t-s)} H^B_s ds\\
    A_t^\top A_t - B_t^\top B_t =& e^{-2\eta\lambda t}(A_0^\top A_0 - B_0^\top B_0) \\
    &- \eta \int_0^t e^{-2\eta\lambda(t-s)} (H^{A\top}_s A_s + A_s^\top H^{A}_s - H^{B\top}_s B_s - B_s^\top H^{B}_s) ds.
\end{align}
\end{proposition}

\begin{proof}
We have \begin{align*}
    d(e^{\mu t} H^A_t) &= e^{\mu t} (\mu H^A_t dt + dH^A_t)\\
    &= \mu e^{\mu t} (G_t B_t dt + \sigma dW^A_t)
\end{align*}
such that
\begin{align*}
    H^A_t &= e^{-\mu t} \left(H^A_0 + \mu \int_0^t e^{\mu s} (G_s B_s ds + \sigma dW^A_s)\right)\\
    &= \mu \int_0^t e^{-\mu (t-s)} G_s B_s ds + \sqrt{\frac{\mu \sigma^2}{2}} W^A_{1-e^{-2\mu t}}
\end{align*}
Note that the second term is an abuse of notation.
The derivation of the integral form of $H^B_t$, $A_t$ and $B_t$ follows the same logic.

For the last two equations, we get

\begin{align*}
    d(A_t^\top A_t - B_t^\top B_t) =& dA_t^\top A_t + A_t^\top dA_t - dB_t^\top B_t - B_t^\top dB_t\\
    =& -\eta\left(\left(H^A_t + \lambda A_t\right)^\top A_t + A_t^\top \left(H^A_t + \lambda A_t\right)\right)\\
    & + \eta\left(\left(H^B_t + \lambda B_t\right)^\top B_t + B_t^\top \left(H^B_t + \lambda B_t\right)\right)\\
    =& -2\eta\lambda (A_t^\top A_t - B_t^\top B_t) - \eta (H^{A\top}_t A_t + A_t^\top H^{A}_t - H^{B\top}_t B_t - B_t^\top H^{B}_t)dt\\
    \coloneqq& -2\eta\lambda (A_t^\top A_t - B_t^\top B_t) + Q_t dt
\end{align*}

which gives
$$A_t^\top A_t - B_t^\top B_t = e^{-2\eta\lambda}(A_0^\top A_0 - B_0^\top B_0) + \int_0^t e^{-2\eta\lambda(t-s)} Q_s ds$$
\end{proof}

Before analyzing the implications of this proposition, let us state a lemma that will allow us to bound the probability of a Brownian motion diverging:
\begin{lemma}\citep{CM_1940__7__283_0}
    \label{lemma:bound_wiener}
    Let $(B_t)$ be a 1D Wiener process. Then, for $t, L > 0$, $\mathbb{P}[\max_{s\in[0, t]}B_s > L] = 2 \mathbb{P}[B_t > L]$.
\end{lemma}

Let us now analyse the consequences of Proposition~\ref{prop:exp_decay}. We will assume that $A_t$ and $B_t$ remain L2-bounded by $M>0$ (which is true for all converging dynamic modulo steady state noise), and that $G_t$ remains L2-bounded by $K$ (either using a Lipschitzian loss, or using clipping). 

First observe that, using lemma \ref{lemma:bound_wiener}, for $\varepsilon > 0$, with probability $1-\varepsilon$, the term $\sqrt{\frac{\mu}{2}} \sigma W^A_{1-e^{-2\mu t}}$ and the correspond $B$ will remain bounded by $\sigma  \sqrt{\mu nd\ln{\frac{4nd}{\varepsilon}}}$.

This way, $H^A$ and $H^B$ will with probability $1-\varepsilon$ remain bounded by $KM + \sigma  \sqrt{\mu nd \ln{\frac{4nd}{\varepsilon}}}$. With that same probability, the term $$\eta \int_0^t e^{-2\eta\lambda(t-s)} (H^{A\top}_s A_s + A_s^\top H^{A}_s - H^{B\top}_s B_s - B_s^\top H^{B}_s)ds$$ will remain bounded by $4 \frac{\eta M}{\lambda} \left(KM + + \sigma  \sqrt{\mu nd\ln{\frac{4nd}{\varepsilon}}} \right)$. This is the same order of magnitude as the stochastic term. Until $A^\top A - B^\top B$ is of that order, it exponentially decays.

\subsubsection{Upper bound of \texorpdfstring{$\left\|\|AB^\top\|_* - \frac{1}{2}(\|A\|^2+\|B\|^2)\right\|$}{the distance between L* and L2 norms}}
Finally, we provide the following general result, bounding the gap between $\|AB^\top\|_*$ and $\frac{1}{2}(\|A\|^2+\|B\|^2)$ by the norm of $A^\top A - B^\top B$.

\begin{proposition} \label{prop:diff_bound}
For any matrices $A,B$, we have
\begin{align*}
    \left|\|AB^\top\|_* - \|A\|_F^2\right| &\leqslant \sqrt{\|A^\top A - B^\top B\|_*} \|A\|_* .
\end{align*}
In particular, 
\begin{align*}
    &\left| \|AB^\top\|_* -\frac{\|A\|_F^2+\|B\|_F^2}{2} \right| \\
    &\hspace{4em}\leqslant \sqrt{\|A^\top A - B^\top B\|_*} \frac{\|A\|_* + \|B\|_*}{2}.
\end{align*}
\end{proposition}

\begin{proof}
Let $Q \coloneqq A^\top A-B^\top B$.
Using singular value decomposition, we write $A=U_A\Sigma_AV_A^\top$ and $B=U_B\Sigma_BV_B^\top$, with $\Sigma_A = \left(\begin{array}{c} S_A \\ 0\end{array}\right)$ and $\Sigma_B = \left(\begin{array}{c} S_B \\ 0\end{array}\right)$. Substituting in the previous equation, we get 

$$V_A S_A^2 V_A^\top = V_B S_B^2 V_B^\top + Q$$
i.e.
\begin{equation}\label{eq:tmp}
V_A S_A V_A^\top = \sqrt{V_B S_B^2 V_B^\top + Q} = V_B (S_B + \Delta) V_B^\top
\end{equation}
where $V_B \Delta V_B^\top \coloneqq \sqrt{V_B S_B^2 V_B^\top + Q} - \sqrt{V_B S_B^2 V_B^\top}$. By the Powers-Stormer inequality \citep{powers1970free}, we have $\|\Delta\|_F^2 = \|V_B \Delta V_B^\top\|_F^2 \leqslant \|Q\|_*$.

From there, we rewrite $A$ as
\begin{align}
    A &= U_A \Sigma_A V_A^{\top} V_B V_B^{\top} = U_A \left(\begin{array}{c} S_A V_A^{\top} V_B \\ 0\end{array}\right) V_B^{\top} \\
    & \stackrel{(\ref{eq:tmp})}{=} U_A \left(\begin{array}{c} V_A^{\top} V_B (S_B + \Delta)\\ 0\end{array}\right) V_B^{\top} \\
    &= U_A \left(\begin{array}{cc} V_A^{\top} V_B & 0 \\ 0 & I \end{array}\right) \left(\Sigma_B + \left(\begin{array}{c} \Delta \\ 0\end{array}\right)\right)V_B^{\top}\label{eq:trick_bal}
\end{align}

Consequently, $\|AB^\top\|_* = \|S_B^2 + \Delta S_B\|_*$ and 
\begin{align*}
    \left|\|AB^\top\|_* - \|B\|_F^2\right| &\leqslant \|\Delta S_B\|_* \leqslant \|\Delta\| \|S_B\|_* \leqslant \|\Delta\|_F \|S_B\|_* \\
    &\leqslant \sqrt{\|Q\|_*} \|B\|_*.
\end{align*}
Similarily, we have $\left|\|AB^\top\|_* - \|A\|_F^2\right| \leqslant \sqrt{\|Q\|_*} \|A\|_*$.

The second inequality is obtained by using the triangle inequality.

\end{proof}

\subsubsection{Upper bound of \texorpdfstring{$\left\|\|\prod_l A_l\|_{2/L}^{2/L} - \frac{1}{L}\sum\|A_l\|_F^2\right\|$}{the distance between Schatten norm and L2 norm}}
We here provide a more general version of Proposition \ref{prop:norm_eq}.
\begin{proposition}
    Let $q \geq r > 0$, $A_1 \in \mathcal{M}_{q,r}$, $A_l\in \mathcal{M}_{r,r}$ for $l\in[2..L-1]$, $A_L \in \mathcal{M}_{q,r}$ and $A = \prod_{l=1}^L A_l$. Let $\varepsilon > 0$. We assume that the sequence $(A_l)_{l\in[1..L]}$ is $\varepsilon$-balanced, i.e. that for $l \in [1.. L-1]$, $$\left\|A_l^\top A_l - A_{l+1} A_{l+1}^\top\right\|_* \leqslant \varepsilon.$$

    Furthermore, assume that
    \begin{equation*}
    \varepsilon \leq \frac{1}{L^4} \min_l \|A_l\|_*^{L}. 
    \end{equation*}
    Then for $k \in [1.. L]$, we have $$\left| \left\|\prod_{l=1}^L A_l\right\|_{2/L}^{2/L} - \|A_k\|_F^2 \right|\leq  r\|A_k\|_*^{L-1}e^{2/L} L^{4/L} \varepsilon^{1/L} $$
    In particular, $$\left | \left\|\prod_{l=1}^L A_l\right\|_{2/L}^{2/L} - \frac{1}{L}\sum_{l=1}^L\|A_l\|_F^2 \right| \leq \frac{r}{L}\sum_l\|A_l\|_*^{L-1}e^{2/L} L^{4/L} \varepsilon^{1/L}$$
\end{proposition}
\begin{proof}
    We will assume for this proof that the $A_l$ are square matrices for all $l\in[1..L]$. This proof holds as it is for the more general case, with more cumbersome notations, and with the trick used in equation \ref{eq:trick_bal}.
    For $l \in [1..L]$, we denote by $A_l = U_l \Sigma_l V_l^\top$ be the SVD of $A_l$.
    Let $k \in [1..L]$. We first show by recurrence that we can write 
    \begin{align*}
        A_l = O_{l-1} (\Sigma_k + \Delta_l) {O_{l}}^\top  && \text{ for } l \in [1..L]\\
    \end{align*}
    for an appropriate choice of orthogonal matrices $(O_l)_{l\in[0..L]}$ and symmetric matrices $(\Delta_l)_{l\in[1..L]}$, and such that : $O_{k-1} = U_k$, $O_k = V_k$ and $\|\Delta_l\|_F\leqslant \sqrt{\varepsilon} |l-k|$ for $l \in [1..k]$.
    The recurrence is symmetric for $l \leqslant k$ and for $l \geqslant k$. We will prove it in the latter case. For $l=k$, the statement holds. Let $l\in [k..L-1]$. We assume we can write $A_l = O_{l-1} (\Sigma_k + \Delta_l) {O_l}^\top$, with $\|\Delta_l\|_F \leqslant \sqrt{\varepsilon} (l-k)$. Let $Q \coloneqq A_l^\top A_l - A_{l+1} A_{l+1}^\top$. We have
    $$O_l (\Sigma_k + \Delta_l)^2 {O_l}^\top = U_{l+1} \Sigma_{l+1}^2 U_{l+1}^\top + Q.$$
    Similar to the previous proof, we can write 
    $$O_l (\Sigma_k + \Delta_l + \Delta) {O_l}^\top = U_{l+1} \Sigma_{l+1} U_{l+1}^\top$$
    with $\Delta$ symmetric verifying $$\|\Delta\|_F^2 \leqslant \|Q\|_* \leqslant \varepsilon.$$
    We can rewrite
    \begin{align*}
        A_{l+1} &=  U_{l+1} \Sigma_{l+1} V_{l+1}^\top\\
        &= O_l {O_l}^\top U_{l+1} \Sigma_{l+1} V_{l+1}^\top\\
        &= O_l (\Sigma_k + \Delta_l + \Delta) {O_l}^\top U_{l+1} V_{l+1}^\top.
    \end{align*}
    We set $O_{l+1} = V_{l+1}U_{l+1}^\top O_l$, $\Delta_{l+1} = \Delta_l + \Delta$ to get the desired result. We verify that $\|\Delta_{l+1}\|_F \leqslant \sqrt{\varepsilon} (l-k+1)$.

    Now that we have proven our lemma, let us observe that
    $$\prod_{l=1}^L A_l = O_0 \left(\prod_l (\Sigma_k + \Delta_l)\right) {O_L}^\top$$
    with $$\left\|\prod_{l=1}^L A_l\right\|_{2/L} = \left\|\prod_l (\Sigma_k + \Delta_l)\right\|_{2/L}.$$
    
    Notice that  $\left\|\Sigma_k^L\right\|_{2/L}^{2/L} = \|A_k\|_F^2$. 
    Using the triangular inequality of $A\rightarrow \|A\|_p^p$ for $0<p\leq 1$, we have 

\begin{align*}
    \left| \left\|\prod_{l=1}^L A_l\right\|_{2/L}^{2/L} - \|A_k\|_F^2 \right| = &\left| \left\|\prod_{l=1}^L (\Sigma_k + \Delta_l)\right\|_{2/L}^{2/L} - \left\|\Sigma_k^L\right\|_{2/L}^{2/L} \right| \\
    \leq & \left\|\prod_{l=1}^L (\Sigma_k + \Delta_l) -\Sigma_k^L \right\|_{2/L}^{2/L} .
\end{align*}

Furthermore, using the fact that $\|AB\|_{2/L}\leq\|A\|_{2/L}\|B\|_F$, we have
\begin{align*}
\left\|\prod_{l=1}^L (\Sigma_k + \Delta_l) -\Sigma_k^L \right\|_{2/L}
\leq & \sum_{l=1}^L \binom{L}{l}\|\Sigma_k^{L-l}\|_{2/L}  L^l\varepsilon^{l/2} \\
\leq & \sum_{l=1}^L \binom{L}{l} r^{L/2} \|\Sigma_k^{L-l}\|_*^{L/2}  L^l\varepsilon^{l/2} \\
\leq & \sum_{l=1}^L \binom{L}{l} r^{L/2} \|\Sigma_k\|_*^{L/2(L-l)}  L^l\varepsilon^{l/2}  \\
\leq &  r^{L/2}(\|\Sigma_k\|_*^{L/2} + L\varepsilon^{1/2})^L - \|\Sigma_k\|_*^{L^2/2}) \\
\leq &  r^{L/2}\|\Sigma_k\|_*^{L^2/2}\left(\left(1 + L\sqrt{\frac{\varepsilon}{\|\Sigma_k\|_*^{L}}}\right)^L - 1\right) \\ 
\leq &  r^{L/2}\|\Sigma_k\|_*^{L^2/2}e L^2 \sqrt{\frac{\varepsilon}{\|\Sigma_k\|_*^{L}}}  \\
\leq &  r^{L/2}\|\Sigma_k\|_*^{L(L-1)/2}e L^2 \varepsilon^{1/2} 
\end{align*}

where in the penultimate line, we used $(1+x)^n-1 \leq enx$ whenever $x<\frac{1}{n}$, and using the assumption of the proposition.

Ultimately, we thus obtain

\begin{align*}
    \left| \left\|\prod_{l=1}^L A_l\right\|_{2/L}^{2/L} - \|A_k\|_F^2 \right| \leqslant r\|A_k\|_*^{L-1}e^{2/L} L^{4/L} \varepsilon^{1/L} 
\end{align*}

which concludes the proof for the first inequality. The second inequality is obtained by using the triangular inequality.
\end{proof}

Note that the above proposition can be used to argue that optimizing a deep linear network of depth $L$ will very quickly co-optimize the $L_p$-Schatten norm of the product, with an exponential time constant $\lambda/L$. In fact, even in a deep linear network, it can be shown that the matrices become exponentially balanced using the same proof as in Lemma~\ref{lemma:exp_decay}. For the assumption $\varepsilon \leq \frac{1}{L^4} \min_l \|A_l\|_*^{L}$ to hold, it suffices that the matrix norm remains lower-bounded by a strictly positive value, a reasonable assumption for loss functions of interest. Note however that this assumption is used only to obtain the convenient upper bound expression, but the exponential decay is trivially true even without it.

\subsection{On the boundness condition of \texorpdfstring{$A$ and $B$}{A and B}}
\label{apx:discussion_boundness}
We now examine the assumption that both $A$ and $B$ are bounded. It can be shown that, under stochastic dynamics with momentum and for certain loss types, $A$ and $B$ will remain bounded with high probability. Below, we present two examples of sufficient conditions on the loss function to ensure this boundedness. Although one could formulate an expanding set of such sufficient conditions to cover a broader class of losses that lead to bounded parameters, there are still practical scenarios where certain losses might not satisfy these conditions. In these cases, practitioners often achieve stable dynamics through hyperparameter tuning, and our theorem remains applicable, as it is designed to accommodate these scenarios rather than exclude them. Therefore, the boundedness assumption is broadly relevant in practice, particularly in training setups that use weight decay, and we leverage this assumption to discuss the rank-regularization effect that impacts such training processes.

\subsubsection{Sufficient conditions}

Henceforth, we will refer by $\theta=(A,B)$.

\paragraph{Sufficient condition 1: Gradient flow with lower bounded loss function}

We consider a the following gradient flow dynamics with weight decay:

\begin{align*}
    \dot{\theta} &= -\eta (\nabla_{\theta}L + \lambda \theta)
\end{align*}

$\eta$ is the learning rate hyperparameter, and $\lambda$ the weight decay strength.

A sufficient condition on the loss is that it is lower bounded, which is the case for most common losses.

Indeed, the above dynamic is the gradient flow dynamic of the loss $L'(\theta) := L(\theta) + \lambda \|\theta\|^2$. Given that $L'$ is also lower bounded, and that $L'(\theta_t)$ is a monotonically decreasing function of time, $L'(\theta_t)$  must converge to a constant real value, i.e. $L(\theta_t) + \lambda \|\theta_t\|^2 \rightarrow_{t\rightarrow\infty} c$ for some $c$. If $\|\theta\|$ is unbounded from above, then necessarily $L$ is unbounded from below, which is a contradiction.

\paragraph{Sufficient condition 2: Gradient flow with momentum with Lipschitz gradient}

We consider a the following gradient flow dynamics with momentum and decoupled weight decay:

\begin{align*}
    \dot{G} &= \mu (\nabla_{\theta}L - G) \\
    \dot{\theta} &= -\eta (G + \lambda \theta)
\end{align*}

where $G$ is the exponential average of the gradient of $\theta$. $\mu$ is the momentum hyperparameter, $\eta$ the learning rate, and $\lambda$ the weight decay strength.

A sufficient condition on the gradient is to be $\min (1, \frac{\eta\lambda}{\mu})$-Lipschitz  with respect to the parameters $\theta$ sufficiently far, i.e. for $ \|\theta\| > P$ for a given $P$:

The momentum makes the analysis of the dynamics more complicated. However, defining $F = \left[\begin{array}{c} \theta \\ G \end{array}\right]$, and $M = \left(\begin{array}{cc}
    \eta\lambda & \eta\\
    0 & \mu
\end{array}\right)$ and $U = \left[\begin{array}{c} 0 \\ \nabla_{\theta}L \end{array}\right]$
one can rewrite the equations as:
$$
\dot{F} = -M F + \mu U
$$
The derivative of the squared norm of $F$ verifies:
\begin{align*}
    \frac{d}{dt}\|F\|^2 &= \Tr \dot{F} F^T\\
    &= -\Tr{MFF^{\top}} + \mu \Tr U^{\top}F\\
    &\leqslant -\min (\eta\lambda, \mu) \|F\|^2 + \mu \Tr U^\top F\\
    &\leqslant -\min (\eta\lambda, \mu) \|F\|^2 + \mu \|\nabla_{\theta}L\|\|F\|\\
    &= \mu\|F\|\left(\|\nabla_{\theta}L\| - \min (1, \frac{\eta\lambda}{\mu} )\|F\|\right)\\
    &\leqslant \mu\|F\|\left(\|\nabla_{\theta}L\| - \min (1, \frac{\eta\lambda}{\mu}) \|\theta\|\right)
\end{align*}

The dynamics of $F$ are flow dynamics. Whenever $\|F\|$ reaches $P$, its norm is decreasing. $\|F\|$ can thus never exceed $P$. As $\|F\|$ is an upperbound on $\|\theta\|$, the same holds for $\|\theta\|$. We also observe that the Lipschitz condition doesn't need to hold for all $\theta > P$. In fact, it suffices that it holds for any borderless submanifold of codim -1 (for example the sphere of radius M) that contains the initialization point.

\paragraph{Note on stochasticity}
To deal with stochasticity, we consider similar equations:
\begin{align*}
    dG &= \mu (\nabla_{\theta}L dt + dW - G dt) \\
    d\theta &= -\eta (G + \lambda \theta) dt
\end{align*}
which become:
$$dF = -MFdt + \mu Udt + \mu dW$$
The integral form is:
$$F = F(0) + \mu e^{-tM} \int e^{sM}Uds + \mu e^{-tM} \int e^{sM}dW $$
The process can diverge because of the stochasticity. However, similarly to proposition B.2, we can fix for any $\delta > 0$ an upper bound on $W$ that holds with probability at least $1-\delta$; this allows to bound the contribution of the stochasticity. We deal with the gradient component similar to the non-stochastic proof.

\subsubsection{Pathological examples where the boundedness does not hold}

An example of a loss for which the parameters will diverge is when we allow losses to be negative, and diverge to minus infinity "stronger" than the weight regularization term.

An obvious, albeit constructed such loss is $L(AB^\top) = -\|AB^\top\|^2$. Then, the gradient of $L$ w.r.t. $A$ (resp. $B$) is $\nabla_A L =  -AB^\top B$ (resp. $\nabla_B L = -BA^\top A$). Even with the decay term, one can see that if $A,B$ are initialized to be e.g. orthogonal matrices scaled by some $\alpha > \lambda$, both $A,B$ will diverge to infinity.

Such negatively unbounded objective functions to be minimized may be found in e.g. the reinforcement learning setting, when using undiscounted returns.

\section{Equilibrium condition of 2-layer Linear network}\label{app:deeplinear}
We assume $\lambda >0$ and that the singular values of $Y^\top X$ are all non-zero and distinct. 

We start with the following set of equations:
\begin{align} \label{eq:root}
\lambda B^\top &= A^{\top}(S - A B^\top)  \\
\lambda A &=  (S - AB^\top )  B 
\end{align}

Clearly, \eqref{eq:root} implies 
\begin{align}
\lambda B^\top B &= A^{\top}(S - A B^\top) B \\
\lambda A^\top A &=  A^\top (S - AB^\top )  B 
\end{align}
and thus that $B^\top B =  A^\top A$.

Furthermore, \eqref{eq:root} also implies 
\begin{align} \label{eq:next}
\lambda A B^\top &= AA^{\top}(S - A B^\top) = AA^{\top}S - AA^{\top}A B^\top \\
\lambda  A B^\top &=  (S - AB^\top ) B B^\top =  S B B^\top - AB^\top B B^\top
\end{align}

Using $B^\top B =  A^\top A$, we get that $AA^{\top}S = S B B^\top$.

Denoting by $a_i, b_i$ the $i$-th row of $A,B$, for any $i,j \in [1..d_{1, 3}]^2$, we have $s_i a_i^\top a_j = s_j b_i^\top b_j$ and $s_j a_j^\top a_i = s_i b_j^\top b_i$. Thus, $\|a_i\|^2 = \|b_i\|^2$, and $a_i^\top a_j = b_i^\top b_j = 0$, since $s_i, s_j$ are distinct and positive. Taken together, $AA^\top$ is a diagonal matrix, which we denote by $D$. We have $D=diag((\|a_i\|^2)_{i \in [1..d_{1, 3}]})$.

In particular, \eqref{eq:next} implies $\lambda A B^\top = D (S - A B^\top) $, i.e. $(\lambda I + D) A B^\top = D S$. Because the entries of $D$ are positive, $(\lambda I + D) $ is invertible, and thus $ A B^\top = (\lambda I + D)^{-1}D S$. In other words, the off-diagonal entries of $A B^\top$ are zero, i.e. $a_i^\top b_j = 0$ for all  $i\neq j$, $i,j \in [1..d_{1, 3}]^2$.

In particular, for a given $i$, we have
\begin{align} 
\lambda a_i^\top b_i &= \|a_i\|^2 (s_i - a_i^\top b_i)\\
\lambda  \|a_i\|^2 &=  (s_i - a_i^\top b_i ) a_i^\top b_i 
\end{align}

Using the positivity of $s_i$ and $\lambda$, one can see that necessarily, $a_i=b_i$.

\section{On the link between the full loss and the restricted loss}
\label{linktwolosses}
In the theoretical section, we study a pruned version of the total losses:
\begin{align*}
    \mathcal{L}_{L2}(A,B, \theta) &\coloneqq L(AB^\top, ) + \frac{\lambda}{2}(\|A\|^2 + \|B\|^2),\\
    \mathcal{L}_*(AB^\top, \theta) &\coloneqq L(AB^\top) + \lambda\|AB^\top\|_*,
\end{align*}
where the remaining parameters are not accounted for. This in fact still accounts for the general case. Indeed:
\begin{itemize}
    \item stationary points of $\mathcal{L}_{L2}$ will be also stationary in $(A, B)$, hence Lemma \ref{lemma:stationary} still holds. In fact, the latter condition suffices.
    \item For theorem \ref{th:main}, the same proof \ref{app:proof_theorem} shows that $(A, B, \theta)$ is a local minima of $\mathcal{L}_{L2}$ iff 1) $(W=AB^\top. \theta)$ is a local minima of $\mathcal{L}_*$ constrained to matrices $W$ of rank $r$ and 2) $A^\top A = B^\top B$
    \item in the proof \ref{proof:expdecay} of theorem \ref{th:diff_loss}, the gradient $G$ of the first lemma now hides a dependence in $\theta$, but the proof still hold as is. In the second lemma, the gradient $G_t$ depends on the time, and is indirectly hiding a dependence on $\theta$.
\end{itemize}

\section{Link between solutions of AdamW and \texorpdfstring{$L2$}{L2}-regularization} \label{app:adamw}

Consider the following dynamic induced by AdamW, with $\lambda >0$:
\begin{align*}
  G_t &\leftarrow \beta_1 \cdot G_{t-1} + (1-\beta_1) \cdot \nabla_W \mathcal{L}(W_t) \\
  B_t &\leftarrow \beta_2 \cdot B_{t-1} + (1-\beta_2) \cdot \nabla_W \mathcal{L}(W_t)^2 \\
  \hat{G}_t &\leftarrow G_t / {(1-\beta_1^t)} \\
  \hat{B}_t &\leftarrow B_t / {(1-\beta_2^t)} \\
  W_{t+1} &\leftarrow W_{t} -\eta \cdot \left( \hat{G}_t / \left({\sqrt{\hat{B}_t } + \varepsilon} \right) + \lambda W_{t} \right)
\end{align*}
where $\eta$ represents the learning rate and $\beta_1, \beta_2$,
$\varepsilon$ are the common hyperparameters of Adam, $W_t$ is the parameter at time $t$, and where the various operations are applied element-wise. 

Note that the term $-\lambda W_t$ stems from weight decay. 

If the dynamic converges, then necessarily, $ G_\infty = \nabla_W \mathcal{L}(W_\infty)$, $B_\infty = (\nabla_W \mathcal{L}(W_\infty))^2$, and thus $\lambda W_\infty = \frac{-\nabla_W \mathcal{L}(W_\infty)}{|\nabla_W \mathcal{L}(W_\infty)|+\epsilon}$. Clearly, this implies that $\lambda |W_\infty| < 1$. If we further assume that $\lambda |W_\infty| \ll 1$, then the condition becomes $\epsilon \lambda W_\infty \approx -\nabla_W \mathcal{L}(W_\infty)$, which is the equilibrium point of a $L2$-regularized loss with regularization strength $ \frac{\epsilon \lambda}{2} $. Thus, the stationary points of the AdamW optimizer can in practice correspond to stationary points of $L2$-regularized loss, and thus the same low-rank inducing solutions can be found.

We show in Fig.~\ref{fig:toy_wd} a toy experiments illustrating the equivalence in the solutions found by AdamW with decay strength $\lambda_{WD}$ and hyperparameter $\epsilon$, with those found by Adam with $L2$-regularization with regularization strength $\lambda_{L2}=\lambda_{WD} \epsilon$. In particular, we illustrate how a factorized parametrization in this setting will still result in solutions that minimize the nuclear norm, even when trained with AdamW. 

\begin{figure*}[ht]
    \centering
    \hspace{0.2cm}
    \begin{minipage}{.33\textwidth}
        \includegraphics[width=0.8 \textwidth]{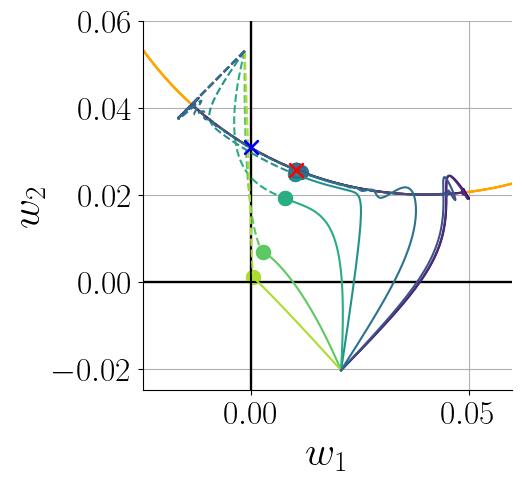}
    \end{minipage}
    \begin{minipage}{.33\textwidth}
       \includegraphics[width=0.8 \textwidth]{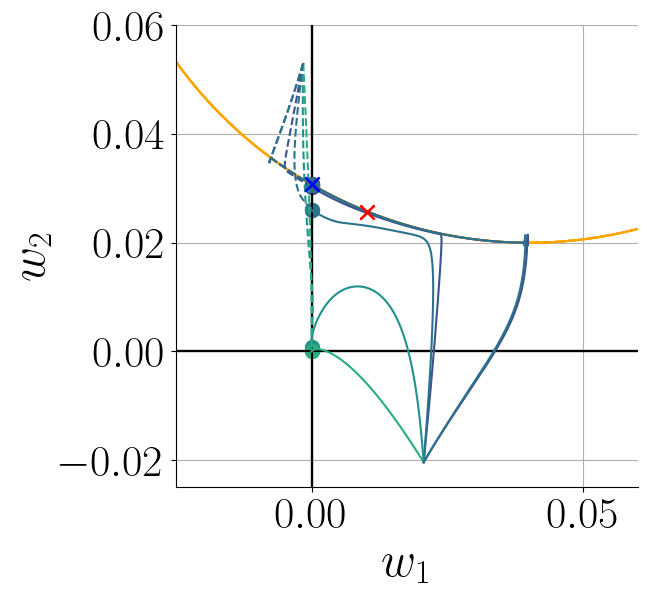}
    \end{minipage}
    \begin{minipage}{.18\textwidth}
       \includegraphics[width=.8 \textwidth]{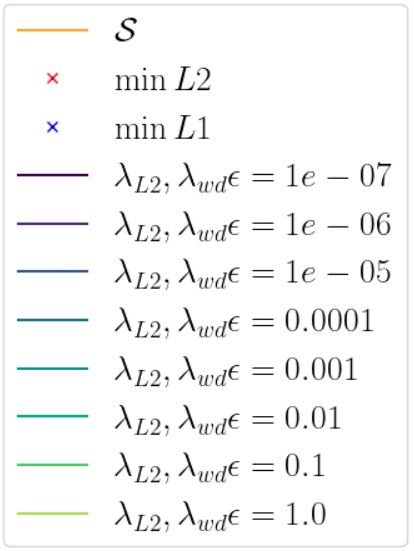}
       \vspace{0.25cm}
    \end{minipage}
    \caption{Trajectory of $w_1, w_2$ in the 2D plane when optimizing the underlying parameter for various hyperparameters. At every coordinate in the plane, the loss is defined as the squared distance to the surface $\mathcal{S}$ in orange. The red (resp. blue) cross represents the points on $\mathcal{S}$ minimizing the $L2$-norm (resp. $L1$-norm). \textit{Left}: $w_1, w_2$ are directly parametrized and optimized by AdamW with decoupled weight decay (in solid line) or Adam with $L2$-regularization (in dotted line). As conjectured, the convergence point of AdamW given the hyperparameter $\epsilon$ and decay strength $\lambda_wd$ corresponds to that of the equilibrium point of the $L2$-regularized loss with regularization strength $\lambda_{L2} = \lambda_wd \epsilon$. 
    \textit{Right}: $w_1, w_2$ are parameterized as a product of two scalars, i.e. $w_1=a_1 b_1, w_2=a_2 b_2$, where $a_1, b_1,a_2, b_2$ are now optimized by AdamW or Adam with $L2$ regularization. Again, the two optimizers find the same convergence point for equivalent hyperparameters. However, the solution found now corresponds to those of the loss regularized by the $L1$-norm of $w_1,w_2$, (corresponding to the nuclear norm for scalars) as predicted.}
    \label{fig:toy_wd}
\end{figure*}

\section{Pretrained foundation models: ViT}

We provide in Figure~\ref{fig:pretrained_vit} the ViT counterpart of Figure~\ref{fig:pretrained}.

\begin{figure*}[ht] 
    \centering    
    \hspace{-0.1cm}
    \begin{minipage}{.2\textwidth}
        \includegraphics[width=1 \textwidth]{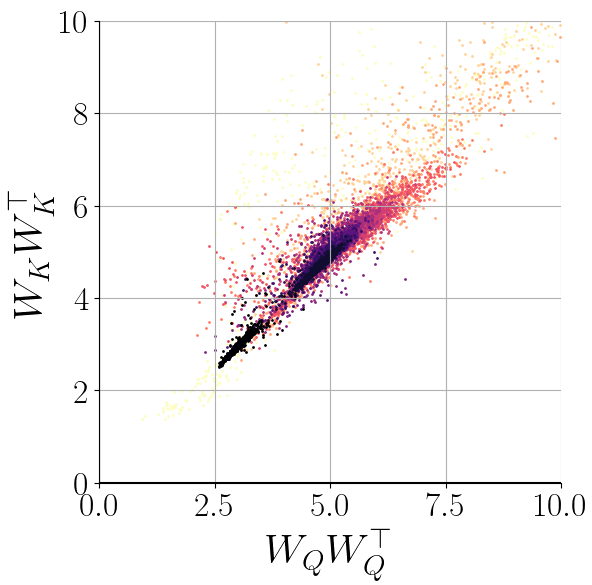}
    \end{minipage}
    \begin{minipage}{.2\textwidth}
       \includegraphics[width=1 \textwidth]{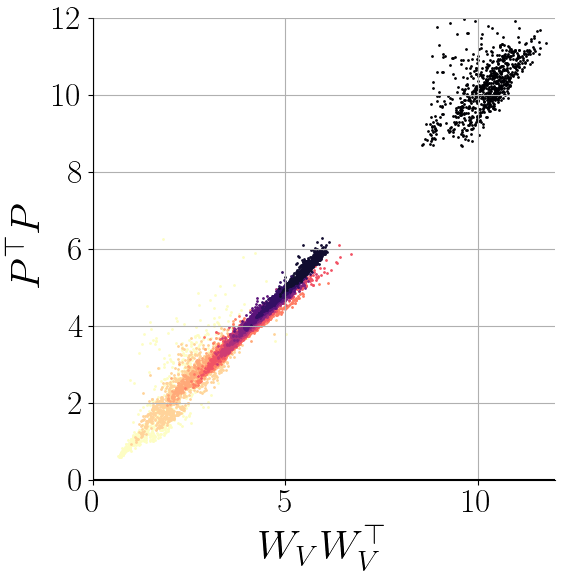}
    \end{minipage}
    \begin{minipage}{.2\textwidth}
        \includegraphics[width=1 \textwidth]{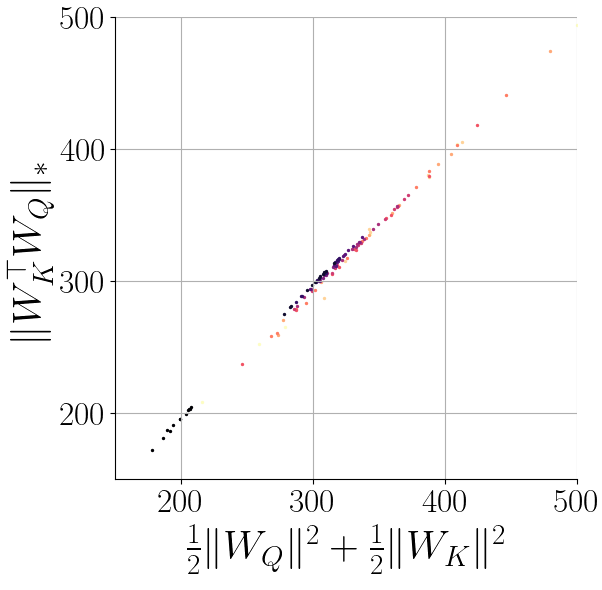}
    \end{minipage}
    \begin{minipage}{.2\textwidth}
       \includegraphics[width=1 \textwidth]{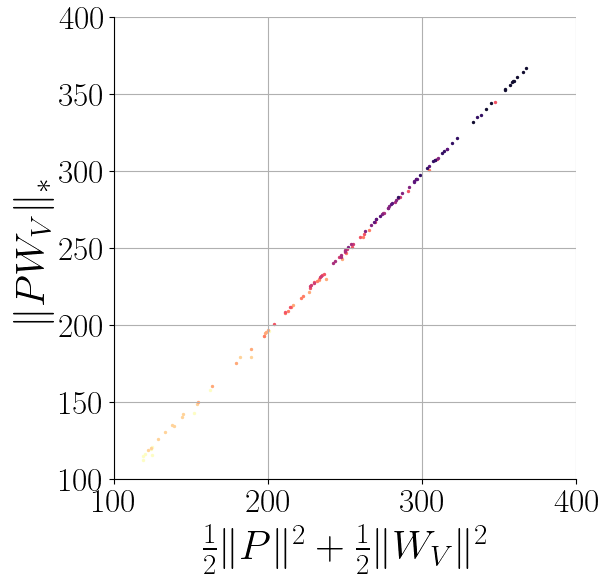}
    \end{minipage}
    \begin{minipage}{.1\textwidth}
    \hspace{0.05cm}
       \includegraphics[width=0.75 \textwidth]{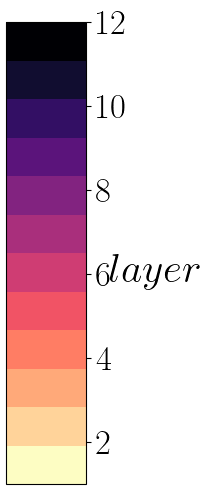}
       \vspace{0.55cm}
    \end{minipage}
    \caption{ 
    Analyses of attention layers in the pre-trained Vision Transformer \citep{wu2020visual}, available on huggingface under the id "google/vit-base-patch16-224-in21k". The leftmost (resp. center left) shows the squared norm of every row of $W_Q$ (resp. $W_V$), for the first head of each layer, against the norm of the corresponding row of $W_K$ (resp. column of $P$). The condition $W_KW_K^\top = W_QW_Q^\top$ would require these norms to be equal, which in fact is mostly true. While the model has not reached a stationary point, this indicates the optimization has advanced enough for this sufficient condition for $\mathcal{L}_*$ to be identical to $\mathcal{L}_{L2}$ to emerge. In fact, the centre right (resp. rightmost) plot shows the scatter plot mapping the Frobenius norm against the nuclear norm for all heads across all layers. The two norms almost perfectly coincide.} 
    \label{fig:pretrained_vit}
\end{figure*}

\newpage

\section{Language modelling experimental details}
\label{app:llm}
Here, we present details of our language modeling experiments, employing standardized values from the literature and consistent, untuned hyperparameters across all trials. Unless specified otherwise, we utilize the conventional GPT-2 transformer architecture with LayerNorm (Ba et al., 2016), incorporating MLPs between self-attention layers and applying skip-connections after each layer. Training is conducted using a standard (autoregressively) masked cross-entropy loss, omitting an input embedding layer but incorporating an output projection before computing logits. Further details can be found in Table \ref{tab:hps-language}.

\begin{table}[bt]
\begin{center}
\caption{Hyperparameters for language modelling experiments.}
\label{tab:hps-language}
\small
\vspace{0.5cm}
\begin{tabular}{@{}p{2.6cm}p{10cm}@{}}
\toprule
Hyperparameter            & Value \\
\toprule
 Dataset & The pile \citep{gao_pile_2020} \\
 Tokenizer & GPT-2 tokenizer - we append a special "EOS" token between every sequence\\
 Context size & 756 \\
 Vocabulary size & 50257 \\
 Vocabulary dim  & 756 \\
 Optimizer & Adam \citep{kingma_adam_2015} with $\epsilon= 1e^{-8}, \beta_1 =0.9, \beta_2=0.95$ \\
 Weight decay & See main text  \\
  Batchsize &  128 \\
 Gradient clipping & Global norm of 1. \\
  Positional encodings &  We add standard positional encodings. \\
   Dropout &  We use an embedding dropout of 0.1 right after adding positional encodings. \\
  Architecture details & 12 layers, 12 heads, key size 64, token size 756, no input- but output-embedding \\
 Weight init &  $W \sim \mathcal{N}(0, \sigma^2)$ with $\sigma= 0.02$ and bias parameter to zero. We scale all \\ & weight matrices before a skip connection with $\frac{1}{2\sqrt{N}}$ with $N$ the number of layers.\\
Learning rate scheduler &  Linear warm-up starting from $1e^{-6}$ to $1e^{-3}$  in the first 8000 training steps, cosine annealing to $10\%$ of the learning rate after warm-up 
for the end of training\\
MLP size & Widening factor 4 i.e. hidden dimension $4*756$ with ReLU \\ & non-linearities \citep{hahnloser_digital_2000}\\
\bottomrule
\end{tabular}
\end{center}
\end{table}

\begin{figure*}
\vspace{1.5cm}
    \begin{minipage}{.245\textwidth}\includegraphics[width=1.\textwidth]{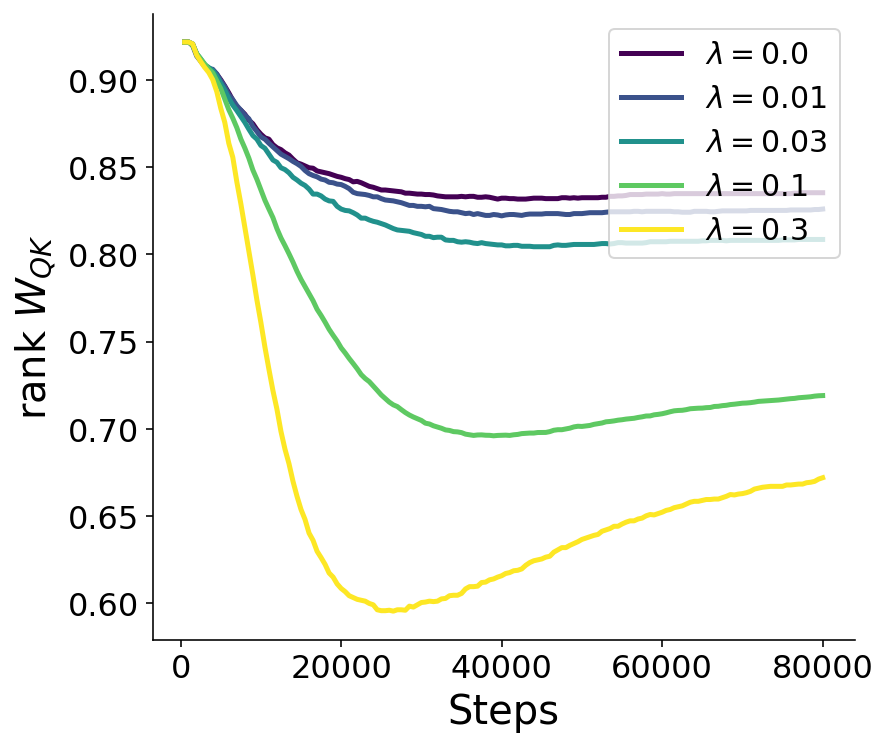}
    \end{minipage}
    \begin{minipage}{.245\textwidth}
    \includegraphics[width=1.\textwidth]{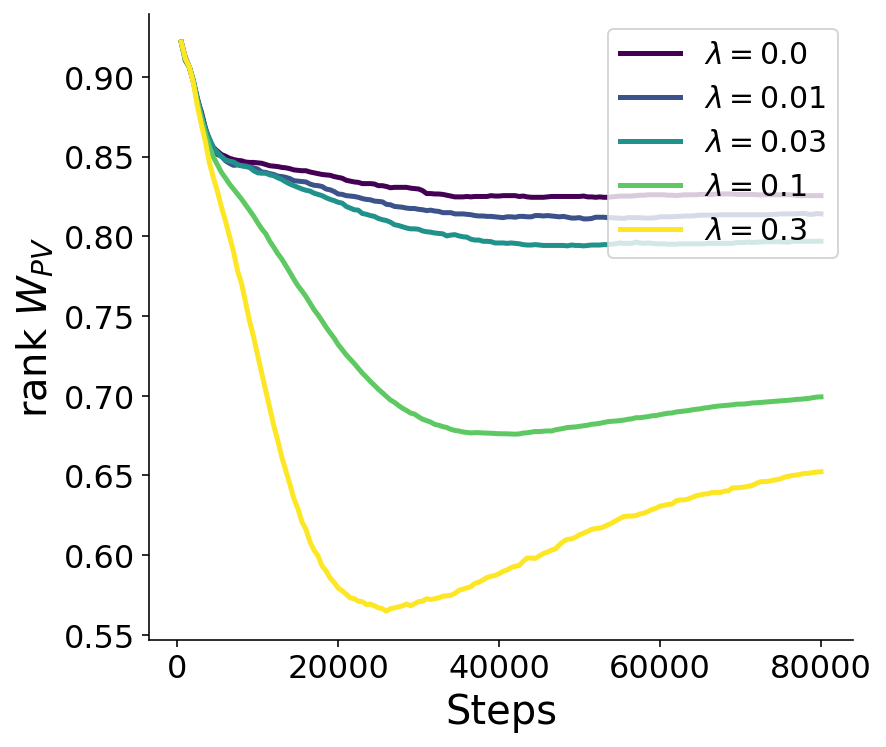}
    \end{minipage}
    \begin{minipage}{.245\textwidth}\includegraphics[width=1.\textwidth]{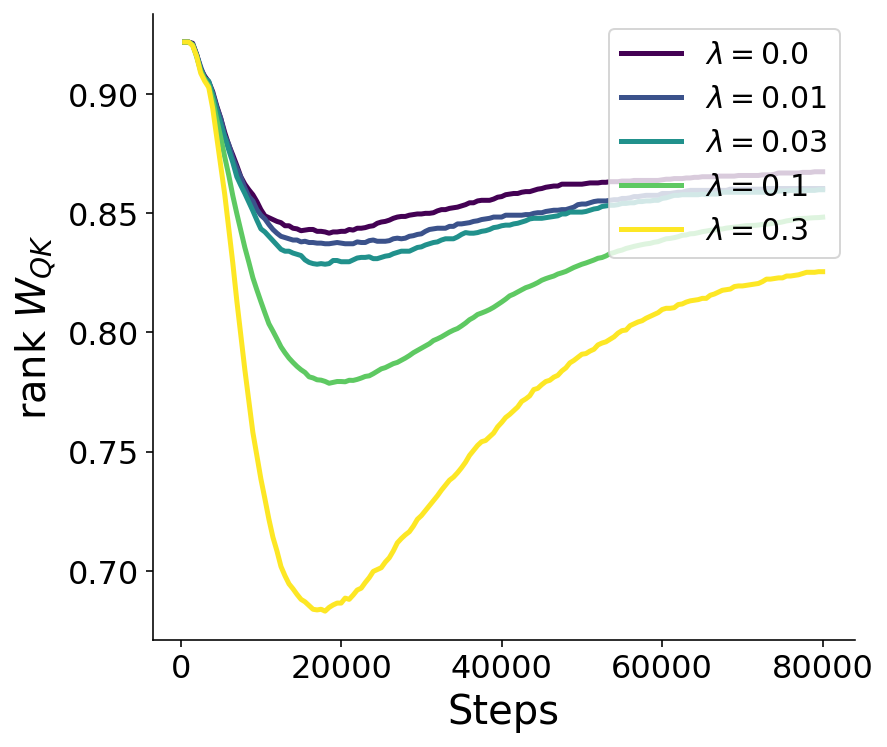}
    \end{minipage}
    \begin{minipage}{.245\textwidth}
    \includegraphics[width=1.\textwidth]{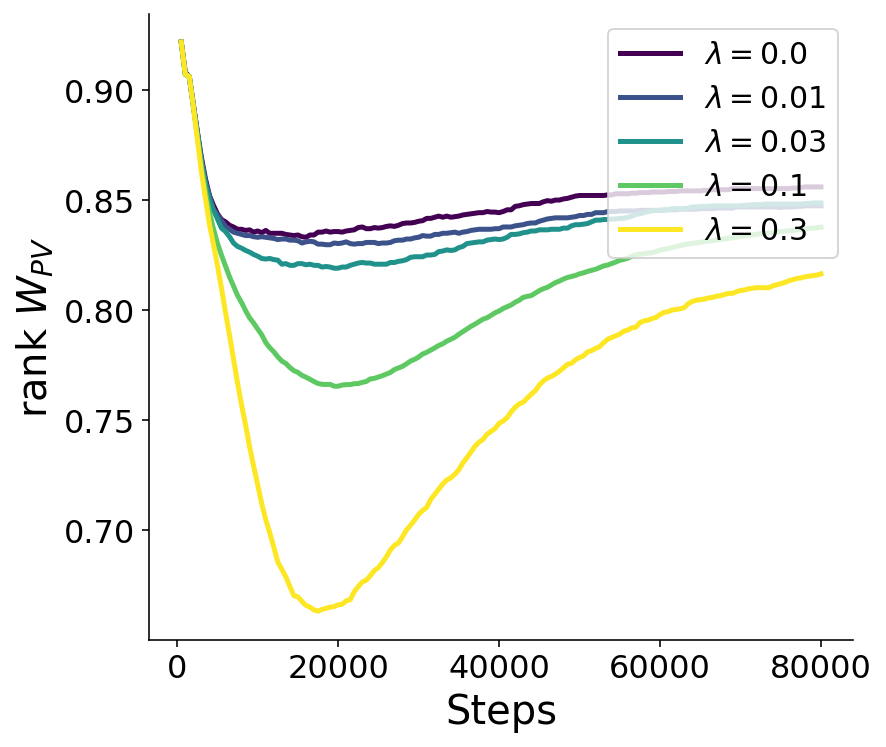}
    \end{minipage}
    \caption{ The rank of weight matrix products $W_K^\top W_Q$ and $PW_V$  averaged across heads of layer 7 (left and outer left) and layer 9 (right and outer right) of an autoregressive transformers trained on the Pile \citep{gao_pile_2020}. For both layers, the decay strength applied to attention layers is varied, while keeping the strength for all other layers fixed. In all cases, we observe again that rank reduction correlates strongly with weight decay strength when optimizing with AdamW.
    }
    \label{fig:llms_app}
    \vspace{1.5cm}
\end{figure*}

\newpage

\section*{NeurIPS Paper Checklist}

\begin{enumerate}

\item {\bf Claims}
    \item[] Question: Do the main claims made in the abstract and introduction accurately reflect the paper's contributions and scope?
    \item[] Answer: \answerYes{}
    \item[] Justification: We provide both theoretical and empirical results supporting the main claim of the paper.

\item {\bf Limitations}
    \item[] Question: Does the paper discuss the limitations of the work performed by the authors?
    \item[] Answer: \answerYes{}
    \item[] Justification: The discussion highlights some of the limitation of this work.
    
\item {\bf Theory Assumptions and Proofs}
    \item[] Question: For each theoretical result, does the paper provide the full set of assumptions and a complete (and correct) proof?
    \item[] Answer: \answerYes{}
    \item[] Justification: All statements are provided with proofs.
    
\item {\bf Experimental Result Reproducibility}
    \item[] Question: Does the paper fully disclose all the information needed to reproduce the main experimental results of the paper to the extent that it affects the main claims and/or conclusions of the paper (regardless of whether the code and data are provided or not)?
    \item[] Answer: \answerYes{}
    \item[] Justification: We hope to provide all hyperparameters and experimental details in the appendix and provide code to reproduce most of the experiments. 

\item {\bf Open access to data and code}
    \item[] Question: Does the paper provide open access to the data and code, with sufficient instructions to faithfully reproduce the main experimental results, as described in supplemental material?
    \item[] Answer: \answerNo{}
    \item[] Justification: We aim to collect the code as soon as possible in a Git repository.
    
\item {\bf Experimental Setting/Details}
    \item[] Question: Does the paper specify all the training and test details (e.g., data splits, hyperparameters, how they were chosen, type of optimizer, etc.) necessary to understand the results?
    \item[] Answer: \answerYes{} 
    \item[] Justification: We provide all experimental details in the main text, as well as the appendix.

\item {\bf Experiment Statistical Significance}
    \item[] Question: Does the paper report error bars suitably and correctly defined or other appropriate information about the statistical significance of the experiments?
    \item[] Answer:\answerNo{}
    \item[] Justification: Many of the experiments were performed on large scale models or foundation models, rendering the computation of multiple seeds unrealistic.

\item {\bf Experiments Compute Resources}
    \item[] Question: For each experiment, does the paper provide sufficient information on the computer resources (type of compute workers, memory, time of execution) needed to reproduce the experiments?
    \item[] Answer: \answerYes{}
    \item[] Justification: We provide an estimate of compute used. 

\item {\bf Code Of Ethics}
    \item[] Question: Does the research conducted in the paper conform, in every respect, with the NeurIPS Code of Ethics \url{https://neurips.cc/public/EthicsGuidelines}?
    \item[] Answer: \answerYes{}
    \item[] Justification: The paper should be considered a theory and/or conceptual paper. We discussed implication for robust machine learning in the main text, and can not anticipate that the presented results can not conform in any aspect with the NeurIPS Code of Ethics.

\item {\bf Broader Impacts}
    \item[] Question: Does the paper discuss both potential positive societal impacts and negative societal impacts of the work performed?
    \item[] Answer: \answerYes{}
    \item[] Justification: The paper shows the importance of choosing a proper weight decay while training transformer.
    
\item {\bf Safeguards}
    \item[] Question: Does the paper describe safeguards that have been put in place for responsible release of data or models that have a high risk for misuse (e.g., pretrained language models, image generators, or scraped datasets)?
    \item[] Answer: \answerNA{}
    \item[] Justification: No data and models realese.

\item {\bf Licenses for existing assets}
    \item[] Question: Are the creators or original owners of assets (e.g., code, data, models), used in the paper, properly credited and are the license and terms of use explicitly mentioned and properly respected?
    \item[] Answer: \answerNA{}

\item {\bf New Assets}
    \item[] Question: Are new assets introduced in the paper well documented and is the documentation provided alongside the assets?
    \item[] Answer: \answerNA{}

\item {\bf Crowdsourcing and Research with Human Subjects}
    \item[] Question: For crowdsourcing experiments and research with human subjects, does the paper include the full text of instructions given to participants and screenshots, if applicable, as well as details about compensation (if any)? 
    \item[] Answer: \answerNA{}
    
\item {\bf Institutional Review Board (IRB) Approvals or Equivalent for Research with Human Subjects}
    \item[] Question: Does the paper describe potential risks incurred by study participants, whether such risks were disclosed to the subjects, and whether Institutional Review Board (IRB) approvals (or an equivalent approval/review based on the requirements of your country or institution) were obtained?
    \item[] Answer: \answerNA{}
    
\end{enumerate}
\end{document}